\documentclass{article}
\usepackage{microtype}
\usepackage{graphicx}
\usepackage{subfigure}
\usepackage{booktabs} % for professional tables

\usepackage{hyperref}

\usepackage[accepted]{icml2018}

\icmltitlerunning{Composite Marginal Likelihood Methods for Random Utility Models}

\usepackage[utf8]{inputenc} % allow utf-8 input
\usepackage[T1]{fontenc}    % use 8-bit T1 fonts
\usepackage{url}            % simple URL typesetting
\usepackage{amsfonts}       % blackboard math symbols
\usepackage{nicefrac}       % compact symbols for 1/2, etc.

\usepackage{times}
\usepackage{natbib}
\usepackage{algorithm}
\usepackage{algorithmic}
\usepackage{amsmath}
\usepackage{amssymb}
\usepackage{mathtools}
\usepackage{multirow}
\usepackage{hyperref}
\usepackage{color}
\usepackage{comment}

\newtheorem{thm}{Theorem}
\newtheorem{dfn}{Definition}

\newtheorem{lem}{Lemma}
\newtheorem{ex}{Example}

\newtheorem{coro}{Corollary}
\newenvironment{sketch}{{\noindent\bf Proof sketch:}\rm }{\hfill $\blacksquare$ }
\newenvironment{proof}{\noindent{\bf {Proof: }}\ }{\hfill$\blacksquare$ \vspace{1mm}}

\newtheorem{claim}{Claim}

\newcommand\Omit[1]{}

\newcommand{\ma}{\mathcal A}

\newcommand{\ml}{\mathcal L}

\newcommand{\mm}{\mathcal M}
\newcommand{\mg}{\mathcal G}
\newcommand{\mw}{\mathcal W}
\newcommand{\ra}{\rightarrow}
\newcommand\rbcml[2]{\text{RBCML}(#1,#2)}
\newcommand\rum{\text{RUM}}

\newcommand\cdf{\text{CDF}}
\newcommand\cll{\text{CLL}_\mm}
\newcommand\exll{\text{ELL}_\mm}
\newcommand\cl{\text{CL}_\mm}
\newcommand\cllpl{\text{CLL}_\text{PL}}
\newcommand\clpl{\text{CL}_\text{PL}}

\begin{document} 

\twocolumn[
\icmltitle{Composite Marginal Likelihood Methods for Random Utility Models}

%\icmlsetsymbol{equal}{*}

\begin{icmlauthorlist}
\icmlauthor{Zhibing Zhao}{rpi}
\icmlauthor{Lirong Xia}{rpi}
\end{icmlauthorlist}

\icmlaffiliation{rpi}{Computer Science Department, Rensselaer Polytechnic Institute, Troy, NY, USA}

\icmlcorrespondingauthor{Zhibing Zhao}{zhaoz6@rpi.edu}
\icmlcorrespondingauthor{Lirong Xia}{xial@cs.rpi.edu}

\icmlkeywords{Composite Marginal Likelihood, Random Utility Models, Plackett-Luce Model, Log-Concavity}

\vskip 0.3in
]

\printAffiliationsAndNotice{}

\begin{abstract} We propose a novel and flexible {\em rank-breaking-then-composite-marginal-likelihood (RBCML)} framework for learning {\em random utility models (RUMs)}, which include the {\em Plackett-Luce} model. We characterize conditions for the objective function of RBCML to be strictly log-concave by proving that strict log-concavity is preserved under convolution and marginalization. We characterize necessary and sufficient conditions for RBCML to satisfy consistency and asymptotic normality. Experiments on synthetic data show that RBCML for Gaussian RUMs achieves better statistical efficiency and computational efficiency than the state-of-the-art algorithm and our RBCML for the Plackett-Luce model provides flexible tradeoffs between running time and statistical efficiency.
\end{abstract}

\section{Introduction}
%In many applications  a decision must be made based on {\em rank data}. For example, in political elections, presidents must be chosen based on the votes. In information retrieval, rankings over documents are combined to a single list~\citep{Liu11:Learning}. In Meta-search engines, multiple rankings over websites produced by different search engines are merged into a single ranking~\citep{Dwork01:Rank}. In recommender systems, rankings over items w.r.t.~various criteria are combined~\citep{Ghosh99:Voting}. And in crowdsourcing, crowdworkers sometimes give rankings as answers, which are aggregated to estimate the correct answer~\citep{Mao13:Better}.

How to model rank data and how to make optimal statistical inferences from rank data are important topics at the interface of statistics, computer science, and economics. {\em Random utility models (RUMs)}~\citep{Thurstone27:Law} are one of the most widely-applied statistical models for rank data. In an RUM, each alternative $a_i$ is parameterized by a utility distribution $\mu_i$.  Agents' rankings are generated in two steps. In the first step, a {\em latent utility} $u_i$ for each alternative $a_i$ is generated from $\mu_i$. In the second step, the alternatives are ranked w.r.t.~their utilities $u_i$ in descending order. The logit model and the probit model, which are very popular in statistics and economics, both have random utility interpretations.
% When the utility distributions are all Gaussians, an RUM can be viewed as an extension of multinomial probit model \zb{re-write this sentence since reviewers don't agree}, which is very popular in statistics and econometrics. 

While providing better fitness to the rank data \citep{Azari12:Random,Zhao18:Learning}, general RUMs are computationally hard to tackle due to the lack of closed-form formulas for the likelihood function. The only known exception is the {\em Plackett-Luce model}~\citep{Plackett75:Analysis,Luce59:Individual}, which is the RUM with Gumbel distributions.  RUMs, especially the Plackett-Luce model, have been widely applied to model and predict human behavior~\citep{McFadden00:Economic}, where the standard case of {\em discrete choice models} can be viewed as the Plackett-Luce model restricted to top choices. Other notable recent applications include elections~\citep{Gormley08:Exploring}, crowdsourcing~\citep{Pfeiffer12:Adaptive}, recommender systems~\citep{Wang2016:Ranking-Oriented}, preference elicitation~\cite{Azari13:Preference,Zhao18:Cost}, marketing~\citep{Berry95:Automobile}, health care~\citep{Bockstael99:The-Use}, transportation~\citep{Bhat07:Flexible}, and security~\citep{Yang11:Improving}.

Recently there has been a growing interest in designing faster and more accurate algorithms for RUMs.  Many algorithms in previous work share the following {\em rank-breaking-then-optimization} architecture. First, rank data are converted to pairwise comparison data. {Second}, based on the pairwise comparisons, various optimization algorithms are designed to estimate the ground truth~\citep{Negahban12:Iterative,Azari13:Generalized,Azari14:Computing,Chen15:Spectral,Khetan16:Data,Khetan16:Computational}. 
% See related work for more discussions.

Pairwise data are often obtained from rank data by applying {\em rank-breaking}, which allows for a smooth tradeoff between computational efficiency and statistical efficiency~\citep{Azari13:Generalized,Azari14:Computing,Khetan16:Data,Khetan16:Computational}. Given $m$ alternatives, a rank-breaking scheme is modeled by a weighted undirected graph $\mg$ (see Figure~\ref{fig:ex} for an example) over $\{1,\ldots,m\}$ (the vertices are positions in a ranking), such that for any ranking $R$ over the $m$ alternatives and any distinct $i_1,i_2\le m$, we obtain $g_{i_1i_2}$ (the weight on the edge $\{i_1, i_2\}$ in $\mg$) pairwise comparisons between alternatives at positions $i_1$ and $i_2$ of $R$.
%copies \zb{confusing about whether they are integers} of $a_{i}\succ a_{i'}$ ($a_i$ is preferred over $a_{i'}$), where $g_{i_1i_2}$ is the nonnegative weight on the edge $\{i_1,i_2\}$ in $\mg$ and $a_{i}$ and $a_{i'}$ are the alternatives ranked at the $i_1$-th and $i_2$-th positions in $R$, respectively.

\noindent{\bf Our Contributions.} By leveraging the celebrated {\em composite marginal likelihood (CML)} methods~\citep{Lindsay1988:Composite,Varin08:Composite}, we propose a novel and flexible {\em rank-breaking-then-CML} framework. % within the rank-breaking-then-optimization architecture. 
Given an RUM, our framework, denoted by $\rbcml{\mg}{\mw}$, is defined by a weighted rank-breaking graph $\mg$ and a CML-weight vector $\mw=\{w_{i_1i_2}:i_1,i_2\le m, i_1\ne i_2\}$, which contains one non-negative weight for each pair of alternatives $(a_{i_1},a_{i_2})$. We note that both $\mg$ and $\mw$ are the algorithm designer's choices. Given rank data $P$, we compute $\vec \theta$ to maximize the following {\em composite log-likelihood} function.
$$
\cll(\vec\theta,P) =\sum_{i_1\neq i_2}(\kappa_{i_1i_2}w_{i_1i_2}\ln p_{i_1i_2}(\vec\theta))
$$
Here $\vec\theta$ represents the parameters of RUM. Given $\mg$, $\kappa_{i_1i_2}$ is the percentage of pairwise comparisons $a_{i_1}\succ a_{i_2}$ in the data. $p_{i_1i_2}(\vec \theta)$ is the probability of $a_{i_1}\succ a_{i_2}$ under RUM with $\vec\theta$, which is the total probability of generating a ranking with $a_{i_1}\succ a_{i_2}$ given $\vec\theta$. We note that the RBCML framework is very general because any combination of $\mg$ and $\mw$ can be used. A breaking graph $\mg$ is {\em uniform}, if all edges have the same weight. Let $\mg_{\text u}$ denote the breaking graph whose weights are all $1$. A CML-weight vector  $\mw$ is {\em symmetric}, if for all $i_1\ne i_2$, we have $w_{i_1i_2}=w_{i_2i_1}$. $\mw$ is {\em uniform}, if all weights are $1$, denoted by $\mw_{\text u}$. 

\noindent{\bf Theoretical contributions.} For convenience we let position-$k$ breaking denote the breaking that consists of all unit-weight edges between position $k$ and all positions after $k$. E.g. the position-$1$ breaking consists of all unit-weight pairwise comparisons in positions $\{(1, 2), (1, 3), \ldots, (1, m)\}$. A weighted union of position-$k$ breakings is a breaking that has the same weight (possibly zero) for each $k$. An example is shown in Figure~\ref{fig:ex}, which is the union of 1/3 position-$1$ breaking and 1/2 position-$2$ breaking. Our theoretical results carry the following message about ``good" RBCMLs.
%In this paper, we aim at theoretically characterizing the conditions on $\mg$ and $\mw$ for ``good" RBCMLs. The main contribution of this paper is the following message.
%we carry a theoretical message on promising and unfavorable algorithms under the RBCML framework.

\noindent{\em We should use $\rbcml{\mg}{\mw}$ with connected and symmetric $\mw$. For Plackett-Luce model, we should use a  breaking $\mg$ that is the weighted union of multiple position-$k$  breakings.
%each of which contains unit-weight edges between $k$ and all $l>k$. 
For RUMs with symmetric utility distributions, we should use $\mg_{\text u}$.} 

The message is established via a series of theorems (Theorems~\ref{thm:logc_conv}, \ref{thm:logconcavemarginal}, \ref{thm:asymptotic}, \ref{thm:cmlrbconsistencypl}, and \ref{thm:cmlrbconsistencyrum}). %We first characterizes {\em strict} log-concavity of  $\cl(\vec\theta,P) $ under the constraint $\theta_m=0$, which is a normalization condition we assume w.l.o.g.~in this paper, because for any $k\in \mathbb R$, $\vec\theta$ and $\vec\theta+k$ have the same distribution over rankings.
%\noindent{\bf Theorem~\ref{thm:logconcavepl} and~\ref{thm:logconcaverum}.} {\em (Informally) For any $\mg,\mw$ under mild conditions, for any data $P$, $\cl(\vec\theta,P)$ is strictly log-concave for (1) the Plackett-Luce model, and (2) any RUM where the PDF of each utility distribution is log-concave, given that $\theta_m=0$.}
%Here $\mg\otimes\mw$ is the weighted graph obtained from multiplying weights on each edge in $\mg$ and $\mw$, respectively. $\mg\otimes\mw$ being connected means that for any pair of vertices $i,i'$, there is a directed path with non-zero weight from $i$ to $i'$. 
%The proofs rely on
Theorems~\ref{thm:logc_conv} and \ref{thm:logconcavemarginal}, which prove that strict log-concavity is preserved under convolution and under marginalization, are of independent  interest.

\noindent{\bf Algorithmic contributions.} Experiments on synthetic data for Gaussian RUMs, where each utility distribution is Gaussian, show that RBCML($\mg_{\text u}, \mw_{\text u}$) achieves better statistical efficiency and computational efficiency than the GMM algorithm by~\citet{Azari14:Computing}. For the Plackett-Luce model, we propose an RBCML with a heuristic $\mw_H$. We compare our RBCML for the Plackett-Luce model with the consistent rank-breaking algorithm by~\citet{Khetan16:Data} and the I-LSR algorithm by~\citet{Maystre15:Fast} via experiments on synthetic data and show that our RBCML provides a tradeoff between statistical efficiency and computational efficiency.

\noindent{\bf Related Work and Discussions.} Our RBCML framework leverages the strengths of rank breaking and CML.
%combines and takes advantages of CML methods and rank-breaking. 
The major advantage of CML is that often marginal likelihood functions are much easier to optimize than the full likelihood function. However, for RUMs, even computing the marginal likelihood may take too much time, as CML needs to count the number of pairwise comparisons between alternatives in the rankings, which takes $O(m^2n)$ time, where $m$ is the number of alternatives and $n$ is the number of rankings. 
Therefore, standard CML becomes inefficient when $m$ or $n$ are large. RBCML overcomes such inefficiency by applying rank-breaking.
%which is a flexible framework for efficiently extracting pairwise comparisons from rank data. 
The computational complexity of rank-breaking can be $O(kmn)$ for any $k\le m$. Often a tradeoff between computational efficiency and statistical efficiency must be made. 

RBCML generalizes the algorithm proposed by~\citet{Khetan16:Data}, which focused on the Plackett-Luce model and whose optimization technique turns out to be CML with $\mw_{\text{u}}$.\footnote{\citet{Khetan16:Data}'s algorithm works for special partial orders. In this paper, we only focus on comparisons between RBCML and their algorithms restricted to  linear orders.} The comparison between RBCML and other related work is summarized in Table~\ref{tab:summary}.

\begin{table*}[htp]
\centering
\scalebox{1}{\begin{tabular}{|r|l|l|l|}
\hline Algorithms & Breaking& Optimization& RUM\\
\hline
\citep{Azari13:Generalized}& Uniform& GMM &Plackett-Luce\\
\hline
\citep{Azari14:Computing}& Uniform& GMM&\begin{tabular}{@{}l}RUMs with sym.~distributions\end{tabular}\\

\hline\citep{Khetan16:Data,Khetan16:Computational}& any & CML($\mw_{\text{u}})$& Plackett-Luce\\
\hline RBCML& any & general CML & \begin{tabular}{@{}l}Plackett-Luce and \\RUMs with sym.~distributions\end{tabular}\\
\hline
\end{tabular}
}
\caption{RBCML vs.  previous work. GMM stands for Generalized Method of Moments. \label{tab:summary}}
\end{table*}

Our theorems on {\em strict} log-concavity of composite likelihood function generalize~\citet{Hunter04:MM}'s result, which was proved for Plackett-Luce with $\mg_{\text u}$ and $\mw_{\text u}$. Our results can be applied to not only other $\mw$'s under Plackett-Luce, but also other RUMs where the PDFs of utility distributions are strictly log-concave, e.g.~Gaussians.  Technically, proving our results for general RUMs is much more challenging due to the lack of closed-form formulas for the likelihood function.  Another line of previous work proved (non-strict) log-concavity for special cases of RBCML~\citep{Azari12:Random,Khetan16:Computational,Khetan16:Data}. Again, our theorems are stronger because (1) our theorems work for a more general class of RBCML, and (2) strict log-concavity is more desirable than log-concavity because the formal implies the uniqueness of the solution. 

The key step in our proofs is the preservation of strict log-concavity under convolution (Theorem~\ref{thm:logc_conv}) and marginalization (Theorem~\ref{thm:logconcavemarginal}). Surprisingly, we were not able to find these theorems in the literature, despite that it is well-known that (non-strict) log-concavity and strong log-concavity are preserved under convolution and marginalization~\citep{Saumard2014:Log-concavity}. Our proofs of Theorems~\ref{thm:logc_conv} and~\ref{thm:logconcavemarginal} are based on a careful examination of the condition for equality in the Pr\'ekopa-Leindler inequality proved by~\citet{Dubuc1977:Critere}. We believe that Theorems~\ref{thm:logc_conv} and~\ref{thm:logconcavemarginal} are of independent interest. %We believe that strict log-concavity is more desirable than log-concavity because it guarantees the uniqueness of the solution.

%We believe that we have also made a number of significant technical contributions w.r.t.~the state of the art. 
%While previous work proved concavity of the likelihood function, strict log concavity was never proved. In fact, the proof is non-trivial and involves proving our key theorems, for which, surprisingly, we were not able to find a proof in the literature, despite that it is know that (non-strict) log-concavity and strong log-concavity are preserved under the two operations. Our proof took 

%Our Theorem~\ref{thm:asymptotic} on consistency and asymptotic normality of RBCML may seem straightforward at the first glance, as it was commonly believed that both hold for CML methods in general. However, they are not as trivial as one may thought, as~\citet{Xu2011:On-the-robustness} commented: ``{\em consistency of the maximum composite likelihood estimator is claimed in several papers, although without detailed proof}''. In fact, 
\citet{Xu2011:On-the-robustness} provided sufficient conditions for general CML methods to satisfy consistency and asymptotic normality. Unfortunately,  some of the conditions by~\citet{Xu2011:On-the-robustness} do not hold for RBCML. Therefore, we derive new proof of consistency and asymptotic normality for RBCML. %See Section~\ref{sec:asymptotic} for more discussions.
%\citet{Azari13:Generalized,Azari14:Computing} characterized all consistent rank-breakings for GMM optimizer for the Plackett-Luce model and RUMs with symmetric utility distributions. 

\citet{Khetan16:Data,Khetan16:Computational} provide sufficient conditions on rank-breakings for CML with $\mw_{\text u}$ to be consistent under the Plackett-Luce model. It is an open question what are all consistent rank-breakings for CML, even with $\mw_{\text u}$. We answer this question for Plackett-Luce (Theorem~\ref{thm:cmlrbconsistencypl}), as well as a large class of other RUMs (Theorem~\ref{thm:cmlrbconsistencyrum}), and for all $\mw$'s.

%Our characterizations of consistent RBCML (Theorem~\ref{thm:cmlrbconsistencypl} and~\ref{thm:cmlrbconsistencyrum}) are different from and are more challenging to prove than those by~\citet{Azari13:Generalized,Azari14:Computing} for two reasons. First, \citet{Azari13:Generalized,Azari14:Computing}'s characterizations are for GMMs, while ours are for CML methods (see Table~\ref{tab:summary}). Second, technically, according to our Theorem~\ref{thm:asymptotic}, consistency conditions for RBCML is much less constrained than that for GMM. Therefore, the ``only if'' direction in our characterization is harder to prove than the ``only if'' direction in the characterization of \citet{Azari13:Generalized,Azari14:Computing}. We developed new techniques to address these challenges.

%Finally, our adaptive RBCML can further improve the statistical efficiency over the method by~\citet{Khetan16:Data}, even though the difference is small for small $m$. This shows the viability of our method, and also provides an experimental justification of the superior performance of the rank-breaking proposed by~\citet{Khetan16:Data} for small $m$.

%Due to the space constraints, we focus on theoretical understanding as the first step. How to efficiently compute the optimal $\mg$ and $\mw$ within RBCML is a promising direction.
%For future work we will deploy an algorithm within the RBCML framework achieving high statistical efficiency and computational efficiency with a fast estimation of optimal $\mg$ and $\mw$.
\section{Preliminaries}
\label{prel}
Let $\mathcal{A}=\{a_1, a_2, \cdots, a_m\}$ denote the set of $m$ alternatives. Let $\mathcal L(\mathcal A)$ denote the set of all linear orders (rankings) over $\mathcal A$. A ranking $R\in\ml(\ma)$ is denoted by $a_{i_1}\succ a_{i_2}\succ\ldots\succ a_{i_m}$, where $a_{i_1}$ is ranked at the top, $a_{i_2}$ is ranked at the second position, etc. We write $a\succ_R b$ if $a$ is ranked higher than $b$ in $R$. Let $P=\{R_1, R_2, \ldots, R_n\}$ denote the collection of $n$ rankings, called a {\em preference profile}. 

\begin{dfn}[Random utility models (RUMs)] A {\em random utility model} $\mathcal M$ over $\ma$ associates each alternative $a_i$ with a utility distribution $\mu_i(\cdot|\vec\theta_i)$. The parameter space is $\Theta = \{\vec\theta=\{\vec\theta_i|i=1, 2, \ldots, m\}\}$. The sample space is $\mathcal L(\mathcal A)^n$. Each ranking is generated i.i.d.~in two steps. First, for each $i\le m$, a latent utility $u_i$ is generated from $\mu_i(\cdot|\vec\theta_i)$ independently, and second, the alternatives are ranked according to their utilities in the descending order. Given a parameter $\vec\theta$, the probability of generating $R=a_{i_1}\succ a_{i_2}\succ \ldots\succ a_{i_m}$ is
\begin{align*}
\Pr\nolimits_{\mm}(R|\vec{\theta})=&\int^\infty_{-\infty}\int^\infty_{u_{i_m}}\cdots\int^\infty_{u_{i_2}}\mu_{i_m}(u_{i_m}|\vec\theta_{i_m})\cdots\\
& \mu_{i_1}(u_{i_1}|\vec\theta_{i_1})du_{i_1}du_{i_2}\cdots du_{i_m}
\end{align*}
\end{dfn}
In this paper, we focus on the {\em location family}, where the shapes of the utility distributions are fixed and each utility distribution $\mu_i$ is only parameterized by its mean, denoted by $\theta_i$. Let $\pi_i$ denote the distribution obtained from $\mu_i(\cdot|\theta_i)$ by shifting the mean to $0$. For the location family, we have $\pi_i(u_i|\theta_i)=\pi(u_i-\theta_i)$. Because shifting the means of all alternatives by the same distance will not affect the distribution of the rankings, {\bf w.l.o.g.~we let $\mathbf{\theta_m=0}$ throughout the paper}. Moreover, we assume that the PDF of each utility distribution is continuous and positive everywhere. We further say that an RUM is {\em symmetric} if the PDF of each utility distribution is symmetric around its mean. We use Gaussian RUMs to denote the RUMs where all utility distributions are Gaussian.

For any combination of $m$ probability distributions $\pi_1,\ldots,\pi_m$ whose means are $0$, we let RUM$(\pi_1,\ldots,\pi_m)$ denote the RUM location family where the shapes of utility distributions are  $\pi_1,\ldots,\pi_m$. For any probability distribution $\pi$ whose mean is $0$, let RUM$(\pi)$ denote the RUM where the shapes of all utility distributions are $\pi$.

Given a profile $P$ and a parameter $\vec\theta$, we have $\Pr_{\mathcal M}(P|\vec\theta) = \prod^n_{j=1}\Pr_{\mathcal M}(R_j|\vec\theta)$. Because all utilities are drawn independently, the probability of pairwise comparison is $\Pr_\mm(a_{i_1}\succ a_{i_2}|\vec\theta)=\int^\infty_{-\infty}\int^\infty_{u_{i_2}}\mu_{i_1}(u_{i_1}|\vec\theta)\mu_{i_2}(u_{i_2}|\vec\theta)du_{i_1}du_{i_2}$.

\begin{ex}[Plackett-Luce model as an RUM] Let $\mu_i(\cdot|\theta_i)$ be the Gumbel distribution where $\mu_i(x_i|\theta_i)=e^{-(x_i-\theta_i)-e^{-(x_i-\theta_i)}}$. For any ranking $R=a_{i_1}\succ a_{i_2}\succ \ldots\succ a_{i_m}$, we have $\Pr_\text{PL}(R|\vec\theta)=\prod^{m-1}_{t=1}\frac {e^{\theta_{i_t}}} {\sum^m_{l=t}e^{\theta_{i_{l}}}}$. 
The probability of $a_{i_1}\succ a_{i_2}$ under the Plackett-Luce model is $\Pr_\text{PL}(a_{i_1}\succ a_{i_2}|\vec\theta)=\frac {e^{\theta_{i_1}}} {e^{\theta_{i_1}}+e^{\theta_{i_2}}}$.
\end{ex}

A {\em weighted (rank-)breaking} $\mg=\{g_{ii'}:i<i'\le m\}$ can be represented by a weighted undirected graph over positions $\{1,\ldots,m\}$, such that for any $g_{ii'}>0$, there is an edge between $i$ and $i'$ whose weight is $g_{ii'}$. We say that $\mg$ is {\em uniform}, if all weights are the same. Let $\mg_\text{u}$ denote the the uniform breaking where all weights are $1$. For any $1\le k\le m-1$, the position-$k$ breaking is the graph where for any $l>k$, there is an edge with weight $1$ between $k$ and $l$.
%is a function, which can be represented by an undirected weighted breaking graph, denoted as $G$. The vertices of a $G$ are $m$ positions in a full ranking. An edge between vertices $l_1$ and $l_2$ in $G$ represents the comparison between the alternatives at positions $l_1$ and $l_2$. The weight of this edge can be arbitrary nonnegative numbers. A weight $0$ is equivalent to no edge between the corresponding vertices. 
For any $\vec \theta\in {\mathbb R}^{m-1}$, any weighted rank-breaking $\mg$, any pair of alternatives $a_{i_1},a_{i_2}$, let $\mg_{a_{i_1}\succ a_{i_2}}(R)=g_{ii'}$ such that $a_{i_1}$ and $a_{i_2}$ are ranked at the $i$th position and the $i'$th position in $R$, respectively. Given a profile $P$, we define
$\kappa_{i_1i_2}=\frac {\sum^n_{j=1} \mg_{a_{i_1}\succ a_{i_2}}(R_j)} n$, and let $\bar\kappa_{i_1i_2} = E[\kappa_{i_1i_2}|\vec\theta]$. We note that $\kappa_{i_1i_2}$ is a function of the preference profile. $\bar\kappa_{i_1i_2}$ is the expected $\kappa_{i_1i_2}$ value for perfect data given $\vec\theta$, which means that it is a function of the ground truth parameter $\vec \theta$. 

\begin{figure}[htp]
\centering
\begin{tabular}{cc}
\includegraphics[trim=0cm 14cm 16.5cm 0cm, clip=true,width=.21\textwidth]{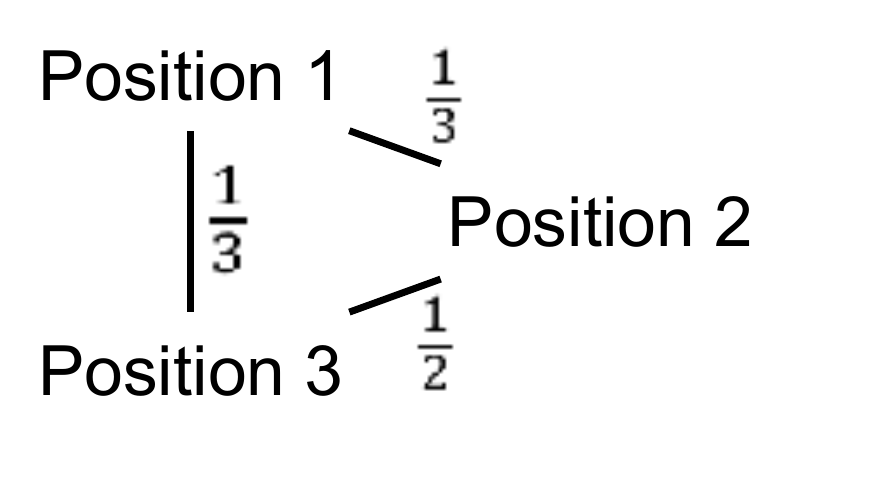}&
\includegraphics[trim=20cm 14cm 0cm 0cm, clip=true,width=.12\textwidth]{rbcml1.pdf}\\
(a) $\mg$. & (b) $\mw$.
\end{tabular}
\caption{A rank-breaking $\mg$ and a CML-weight vector $\mw$. \label{fig:ex}}
\end{figure}

%$V_{\vec\theta,G}(a\succ b)$ denote the expected number of times for $a\succ b$ after applying $G$ when the ground truth is $\vec \theta$. That is, $V_{\vec\theta,G}(a\succ b)=E_{R\sim\Pr_{\vec\theta}}G_{a\succ b}(R)$, where 

%\begin{dfn} (Composite marginal likelihood)
%Given a statistical model $\mathcal M(\vec\theta)$ and a set of marginal events $\mathcal E = \{E_1, E_2, \ldots, E_q\}$ with likelihoods $L(E_k|\vec\theta)$, a composite marginal likelihood is the weighted product
%\begin{equation*}
%\text{CL}(\mathcal E|\vec\theta) = \prod^q_{k=1}L(E_k|\vec\theta)^{w_k},
%\end{equation*}
%where $w_k$-s are nonnegative weights.
%\end{dfn}
{
\begin{ex}\label{ex:kappa} Let $m=3, n=2$. The profile $P=\{a_1\succ a_2\succ a_3, a_3\succ a_2\succ a_1\}$.
Let $\mg=\{g_{12}=g_{13}=\frac 1 3, g_{23}=\frac 1 2\}$ as shown in Figure~\ref{fig:ex} (a). Then we have $\kappa_{12}=\kappa_{13}=\frac 1 3 /n=\frac 1 6$, $\kappa_{23}=\frac 1 2/n = \frac 1 4$, $\kappa_{32}=\kappa_{31}=\frac 1 3/n=\frac 1 6$, $\kappa_{21} = \frac 1 2/n=\frac 1 4$. %We can use a matrix $K$ to compactly present the $\kappa$'s.
%\begin{equation*}
%K=\begin{bmatrix}
%0 & 1/6 & 1/6 \\
%1/4 & 0 & 1/4 \\
%1/6 & 1/6 & 0
%\end{bmatrix}
%\end{equation*}
\end{ex}
}
%\lirong{Give an example of G for 3 alternatives, a profile of two rankings abc cba, and how to compute kappas}
\section{Composite Marginal Likelihood Methods}
Let $\mw=\{w_{ii'}:a_i,a_{i'}\in \ma \}$ denote a CML-weight vector. We say that $\mw$ is {\em symmetric}, if for any pair of alternatives $a_i,a_{i'}$, we have $w_{ii'}=w_{i'i}>0$. We say that $\mw$ is {\em uniform}, if all $w_{ii'}$'s are equal. Let  $\mw_\text{u}$ denote a uniform $\mw$. 

We note that vertices in $\mw$ corresponds to the alternatives while vertices in $\mg$ corresponds to positions in a ranking. For example, vertex $i$ in $\mw$ corresponds to $a_i$, while vertex $i$ in $\mg$ corresponds to the $i$th position in a ranking.

{
\begin{ex}\label{ex:w} A symmetric $\mw$ is shown in Figure~\ref{fig:ex} (b), where $w_{12}=w_{21}=1$ and $w_{23}=w_{32}=2$.
%Elements in $\mw$ can be compactly represented in a matrix, where $w_{i_1i_2}$ is at the $(i_1, i_2)$ position. For example, $\mw$ is
%\begin{equation*}
%\mw=
%\begin{bmatrix}
%0 & 1 & 0\\
%1 & 0 & 2\\
%0 & 2 & 0
%\end{bmatrix}
%\end{equation*}
\end{ex}
}
Given $\mg$ and $\mw$, we propose the {\em rank-breaking-then-CML} framework for RUMs, denoted by $\rbcml{\mg}{\mw}$, to be the maximizer of composite log-marginal likelihood, which is defined below.
\begin{dfn}[Composite marginal likelihood for RUMs] Given an RUM $\mm$, for any preference profile $P$ and any $\theta$, let $p_{i_1i_2}(\vec\theta)=\Pr_\mm(a_{i_1}\succ a_{i_2}|\vec\theta)$. The composite marginal likelihood is
%\begin{equation*}%\label{clrum}
$\cl(\vec\theta,P)=\prod_{i_1\neq i_2}(p_{i_1i_2}(\vec\theta))^{\kappa_{i_1i_2}w_{i_1i_2}}$.
%\end{equation*}
The composite log-marginal likelihood becomes: 
\begin{equation}\label{eq:cllrum}
\cll(\vec\theta,P) =\sum_{i_1\neq i_2}\kappa_{i_1i_2}w_{i_1i_2}\ln p_{i_1i_2}(\vec\theta)
\end{equation}
\end{dfn}
%The first order conditions for $\cll$ are the following.
%\begin{equation}\label{rumdr}
%\frac {\partial \cll(\vec\theta,P)} {\partial\theta_i} = \sum_{i'\neq i}(\frac {\kappa_{ii'}w_{ii'}} {p_{ii'}(\vec\theta)}\frac {\partial p_{ii'}(\vec\theta)} {\partial\theta_i}+\frac {\kappa_{i'i}w_{i'i}} {p_{i'i}(\vec\theta)}\frac {\partial p_{i'i}(\vec\theta)} {\partial\theta_i})
%\end{equation}
We let $\rbcml{\mg}{\mw}(P)=\arg\max_{\vec\theta}\cll(\vec\theta,P)$. For the Plackett-Luce model the composite (log-)marginal likelihood has a closed-form formula.
%Then we have
%\begin{align*}
%\frac {\partial^2 \cll(\vec\theta)} {\partial\theta^2_i} &= \sum_{i'\neq i}(\frac {\partial^2 p_{ii'}(\vec\theta)} {\partial\theta^2_i}(\frac {\kappa_{ii'}w_{ii'}} {p_{ii'}(\vec\theta)}-\frac {\kappa_{i'i}w_{i'i}} {p_{i'i}(\vec\theta)})-(\frac {\partial p_{ii'}(\vec\theta)} {\partial\theta_i})^2(\frac {\kappa_{ii'}w_{ii'}} {p_{ii'}(\vec\theta)^2}-\frac {\kappa_{i'i}w_{i'i}} {p_{i'i}(\vec\theta)^2}))\\
%\frac {\partial^2 \cll(\vec\theta)} {\partial\theta_i\partial\theta_{i'}} &= \sum_{i'\neq i}(\frac {\partial^2 p_{ii'}(\vec\theta)} {\partial\theta_i\partial\theta_{i'}}(\frac {\kappa_{ii'}w_{ii'}} {p_{ii'}(\vec\theta)}-\frac {\kappa_{i'i}w_{i'i}} {p_{i'i}(\vec\theta)})-\frac {\partial p_{ii'}(\vec\theta)} {\partial\theta_i}\frac {\partial p_{ii'}(\vec\theta)} {\partial\theta_{i'}}(\frac {\kappa_{ii'}w_{ii'}} {p_{ii'}(\vec\theta)^2}-\frac {\kappa_{i'i}w_{i'i}} {p_{i'i}(\vec\theta)^2}))
%\end{align*}
\begin{dfn}[CML for Plackett-Luce] For any $\vec\theta$ and preference profile $P$, the composite marginal likelihood for the Plackett-Luce model is $\clpl(\vec\theta,P)=\prod_{i_1 < i_2}(\frac {e^{\theta_{i_1}}} {e^{\theta_{i_1}}+e^{\theta_{i_2}}})^{\kappa_{i_1i_2}w_{i_1i_2}}(\frac {e^{\theta_{i_2}}} {e^{\theta_{i_1}}+e^{\theta_{i_2}}})^{\kappa_{i_2i_1}w_{i_2i_1}}$. 
The composite log-marginal likelihood is

\begin{align}
\cllpl(&\vec\theta,P) = \sum_{i_1 < i_2}(\kappa_{i_1i_2}w_{i_1i_2}\theta_{i_1}+\kappa_{i_2i_1}w_{i_2i_1}\theta_{i_2}\notag\\
&-(\kappa_{i_1i_2}w_{i_1i_2}+\kappa_{i_2i_1}w_{i_2i_1})\ln(e^{\theta_{i_1}}+e^{\theta_{i_2}}))\label{eq:cllpl}
\end{align}
\end{dfn}
The first order conditions are, for all $i$, 
$\frac {\partial \cllpl(\vec\theta,P)} {\partial\theta_i} =\sum_{i'\neq i}(\kappa_{ii'}w_{ii'}-(\kappa_{ii'}w_{ii'}+\kappa_{i'i}w_{i'i})\frac {e^{\theta_i}} {e^{\theta_i}+e^{\theta_{i'}}}).
$

%Further for any $i'\neq i$ we have
%\begin{align*}\label{plhess}
%\frac {\partial^2 \cllpl(\vec\theta)} {\partial\theta^2_i} &= \sum_{i'\neq i}(-(\kappa_{ii'}w_{ii'}+\kappa_{i'i}w_{i'i})\frac {e^{\theta_i+\theta_{i'}}} {(e^{\theta_i}+e^{\theta_{i'}})^2})\\
%\frac {\partial^2 \cllpl(\vec\theta)} {\partial\theta_i\partial\theta_{i'}} &= \sum_{i'\neq i}(\kappa_{ii'}w_{ii'}+\kappa_{i'i}w_{i'i})\frac {e^{\theta_i+\theta_{i'}}} {(e^{\theta_i}+e^{\theta_{i'}})^2}
%\end{align*}
{
\begin{ex}\label{ex:cllpl}
Continuing Example~\ref{ex:kappa} and Example~\ref{ex:w}, 
\begin{align*}
\cllpl(\vec\theta,P) 
%= (\kappa_{12}w_{12}\theta_1+\kappa_{21}w_{21}\theta_2-(\kappa_{12}w_{12}+\kappa_{21}w_{21})\ln(e^{\theta_1}+e^{\theta_2}))\\
%&+ (\kappa_{13}w_{13}\theta_1+\kappa_{31}w_{31}\theta_3-(\kappa_{13}w_{13}+\kappa_{31}w_{31})\ln(e^{\theta_1}+e^{\theta_3}))\\
%&+ (\kappa_{23}w_{23}\theta_2+\kappa_{32}w_{32}\theta_3-(\kappa_{23}w_{23}+\kappa_{32}w_{32})\ln(e^{\theta_2}+e^{\theta_3}))\\
&=\frac 1 6\theta_1+\frac 1 4\theta_2-(\frac 1 6+\frac 1 4)\ln(e^{\theta_1}+e^{\theta_2})\\
&+\frac 1 2\theta_2-(\frac 1 2+\frac 1 3)\ln(e^{\theta_2}+1)
\end{align*}
By solving the first order conditions, we have $e^{\theta_1}=1$ and $e^{\theta_2}=1.5$. So the outcome of RBCML is $\theta_1=0$, $\theta_2=\ln 1.5$. We recall that $\theta_3=0$  in this paper.
\end{ex}
}
%\lirong{give an example of $\cllpl$ using the G,W, and profile in previous examples.}

\section{Preservation of Strict Log-Concavity}
\begin{dfn}[Log-concavity and strict log-concavity]
A function $f(\vec x)>0$ is {\em log-concave} if $\forall 0<\lambda<1$, we have
$
 f(\lambda\vec x+(1-\lambda)\vec y)\geq f(\vec x)^\lambda f(\vec y)^{1-\lambda}
 $. 
If the inequality is always strict, then $f$ is {\em strictly log-concave}.
\end{dfn}
\begin{thm}[Preservation under convolution]\label{thm:logc_conv} 
Let $f(x)$ and $g(x)$ be two continuous and strictly log-concave functions on $\mathbb R$. Then $f*g$ is also strictly log-concave.
% on $\mathbb R$.
\end{thm}
\begin{proof}
The proof is done by examining the equality condition for the Pr\'ekopa-Leindler inequality. %which is the key step in the proof of the preservation of log-concavity under marginalization. 
Let $h=f*g$, namely, for any $y\in \mathbb R$, $h(y)=\int_{\mathbb R} f(y-x)g(x) dx$.
Because $f$ and $g$ are continuous, so does $h$. To prove the strict log-concavity of $h$, 
it suffices to prove that for any different $y_1,y_2\in\mathbb R$,  $h(\frac{y_1+y_2}{2})>\sqrt {h(y_1)h(y_2)}$.

Suppose for the sake of contradiction that this is not true. Since log-concavity preserves under convolution \citep{Saumard2014:Log-concavity}, $h$ is log-concave. So, there exist $y_1<y_2$ such that $h(\frac{y_1+y_2}{2})=\sqrt {h(y_1)h(y_2)}$. Let $\Lambda(x,y)=f(y-x)g(x)$. %It follows that $H$ is a log-concave function on ${\mathbb R}^2$. 
We further define 
\begin{align*}
H(x)&=\Lambda(x,\frac{y_1+y_2}{2})=f(\frac{y_1+y_2}{2}-x)g(x)\\
F(x)&=\Lambda(x,y_1)=f(y_1-x)g(x)\\
G(x)&=\Lambda(x,y_2)=f(y_2-x)g(x)
\end{align*}
Because (non-strict) log-concavity is preserved under convolution, $\Lambda(x,y)$ is log-concave. We have that for any $x\in\mathbb R$, $H(x)\ge \sqrt{F(x)G(x)}$. The Pr\'ekopa-Leindler inequality asserts that 
\begin{equation}\label{eq:pl}
\int_{\mathbb R}H(x) dx\ge \sqrt {\int_{\mathbb R}F(x) dx\int_{\mathbb R}G(x) dx}
\end{equation}

Because $h(\frac{y_1+y_2}{2})=\int_{\mathbb R}H(x) dx$, $h(y_1)=\int_{\mathbb R}F(x) dx$, $h(y_2)=\int_{\mathbb R}G(x) dx$, and  $h(\frac{y_1+y_2}{2})=\sqrt {h(y_1)h(y_2)}$, (\ref{eq:pl}) becomes an equation. It was proved by~\citet{Dubuc1977:Critere} that: there exist $a>0$ and $b\in \mathbb R$ such that the following conditions hold almost everywhere for $x\in \mathbb R$ (see the translation of Dubuc's result in English by~\citet{Ball2010:Stability}). 1.~$F(x)=aH(x+b)$, 2.~$G(x)=a^{-1}H(x-b)$.

The first condition means that for almost every $x\in\mathbb R$,
\begin{align}
&f(y_1-x)g(x)=af(\frac{y_1+y_2}{2}-x-b)g(x+b)\notag\\
&\Longleftrightarrow \frac{g(x)}{g(x+b)}=a \frac{f(\frac{y_1+y_2}{2}-x-b)}{f(y_1-x)}\label{eq:eqc1}
\end{align}
The second condition means that for almost all  $x\in\mathbb R$,
$
f(y_2-x)g(x)=a^{-1}f(\frac{y_1+y_2}{2}-x+b)g(x-b)
\Longleftrightarrow \frac{g(x-b)}{g(x)}=a \frac{f(y_2-x)}{f(\frac{y_1+y_2}{2}-x+b)}
$. Therefore, for almost all $x\in \mathbb R$, 

\begin{equation}\label{eq:eqc3}\frac{g(x)}{g(x+b)}=a \frac{f(y_2-x-b)}{f(\frac{y_1+y_2}{2}-x)}\end{equation}

Combining (\ref{eq:eqc1}) and (\ref{eq:eqc3}), for almost every $x\in \mathbb R$ we have
\begin{equation}\label{eq:eqc4}
\frac{g(x)}{g(x+b)}=a \frac{f(y_2-x-b)}{f(\frac{y_1+y_2}{2}-x)}=a \frac{f(\frac{y_1+y_2}{2}-x-b)}{f(y_1-x)}
\end{equation}
Because $f(x)$ is strictly log-concave, for any fixed $c\ne 0$, $\frac{f(x+c)}{f(x)}$ is strictly monotonic. Because $y_1\ne y_2$ and 
$y_2-x-b-(\frac{y_1+y_2}{2}-x)=\frac{y_1+y_2}{2}-x-b-(y_1-x)=\frac{y_2-y_1}{2}-b$, we must have that $\frac{y_2-y_1}{2}-b=0$, namely $b=\frac{y_2-y_1}{2}$. Therefore, (\ref{eq:eqc4}) becomes $\frac{g(x)}{g(x+\frac{y_2-y_1}{2})}=a$ for almost every $x\in\mathbb R$, which contradicts the strict log-concavity of $g$. This means that $h=f*g$ is strictly log-concave.
\end{proof}

\begin{thm}[Preservation under marginalization]\label{thm:logconcavemarginal}
Let $h(x,y)$ be a strictly log-concave function on $\mathbb R^2$. Then $\int_{\mathbb R}h(x,y) dx$ is strictly log-concave on $\mathbb R$.
\end{thm}
Again, the proof is done by examining the equality condition for the Pr\'ekopa-Leindler inequality. All missing proofs can be found in the supplementary material.

\section{Strict Log-Concavity of CML}

%\subsection{Strict Log-Concavity of Composite Likelihood for Plackett-Luce Model}
For any profile $P$, let $G(P)$ denote the weighted directed graph where each represents an alternative. For any $1\le i\neq i'\le m$, the weight on the edge from $i$ to $i'$ is $\kappa_{ii'}$. A weighted directed graph is {\em (weakly) connected}, if after removing the directions on all edges, the resulting undirected graph is connected. A weighted directed graph is {\em strongly connected}, if there is a directed path with positive weights between any pair of vertices. Given any pair of weighted graphs $G_1$ and $G_2$, we let $G_1\otimes G_2$ denote the weighted graph where the weights on each edge is the multiplication of the weights of same edge in $G_1$ and $G_2$.
\begin{thm}\label{thm:logconcavepl}
Given any profile $P$, the composite likelihood function for Plackett-Luce, i.e.~$\clpl(\vec\theta,P)$, is strictly log-concave if and only if $\mw\otimes G(P)$ is weakly connected.  $\arg\max_{\vec\theta}\clpl(\vec\theta,P)$ is bounded if and only if $\mw\otimes G(P)$ is strongly connected.
\end{thm}
The proof is similar to the log-concavity of likelihood for BTL by~\citep{Hunter04:MM}. For general RUMs we prove a similar theorem.
\begin{thm} \label{thm:logconcaverum}
Let $\mm$ be an RUM where the CDF of each utility distribution is strictly log-concave. 
Given any profile $P$, the composite likelihood function for $\mm$, i.e.~$\cl(\vec\theta,P)$, is strictly log-concave if and only if $\mw\otimes G(P)$ is weakly connected.  $\arg\max_{\vec\theta}\cl(\vec\theta,P)$ is bounded if and only if $\mw\otimes G(P)$ is strongly connected.
\end{thm}
\begin{sketch} It is not hard to check that when $\mw\otimes G(P)$ is not connected, there exist $\vec\theta^{(1)}$ and $\vec\theta^{(2)}$ such that for any $0<\lambda<1$ we have $\cllpl(\vec\theta^{(1)},P)=\cllpl(\vec\theta^{(2)},P)=\lambda\cllpl(\vec\theta^{(1)},P)+(1-\lambda)\cllpl(\vec\theta^{(2)},P)$, which violates strict log-concavity. Suppose $\mw\otimes G(P)$ is weakly connected, it suffices to  prove for any $i_1\neq i_2$, $\Pr(a_{i_1}\succ a_{i_2}|\vec\theta)$ is strictly log-concave. We can write this as an integral over $u_{i_2}-u_{i_1}$: $\Pr(u_{i_1}>u_{i_2}|\vec\theta) = \int^\infty_0\Pr(u_{i_2}-u_{i_1}=s|\vec\theta)ds$.

Let $\pi^\ast_{i_2}(\cdot|\vec\theta)$ denote the flipped distribution of $\pi_{i_2}(\cdot|\vec\theta)$ around $x=s$, then we have $\pi^*_{i_2}(s-x|\vec\theta)=\pi_{i_2}(s+x|\vec\theta)$. Further we have
$\Pr(u_{i_1}>u_{i_2}|\vec\theta) = \int^\infty_0\int^\infty_{-\infty} \pi_{i_1}(x|\theta_{i_1})\pi_{i_2}(x+s|\theta_{i_2})dxds = \int^\infty_0 \pi_{i_1} * \pi_{i_2}^\ast ds$.

By Theorem~\ref{thm:logc_conv}, $\pi_{i_1} * \pi_{i_2}^\ast$ is strictly log-concave. 
Then we prove that tail probability of a strictly log-concave distribution is also strictly log-concave.
%We only need to prove that tail probability of a strictly log-concave distribution is also strictly log-concave, which is shown in the following lemma. 
%\begin{lem}\label{lem:tailconc}
%Let $f(x)$ be a continuously strictly log-concave differentiable probability density function with support $(-\infty, +\infty)$. $F(x)=\int^x_{-\infty}f(t)dt$ is strictly log-concave.
%\end{lem}
%\begin{proof} The proof is slightly modified from~\citep{Bagnoli2005:Log-Concave}. 
%%Log-concave probability and its applications. Bagnoli, Mark; Bergstrom Ted.
%We will prove $\frac {\partial^2\ln F(x)} {\partial x^2}=\frac {d} {dx} (\frac {f(x)} {F(x)})=\frac {f'(x)F(x)-f(x)^2} {F(x)^2}<0$. Since $F(x)>0$, we only need to prove $
%f'(x)F(x)-f(x)^2<0
%$.
%
%Because $f(x)$ is strictly log-concave, we have that $\frac {d\ln f(x)} {dx} = \frac {f'(x)} {f(x)}$ 
%is decreasing for any $x\in\mathbb R$. So we have $\frac {f'(x)} {f(x)}F(x)=\frac {f'(x)} {f(x)}\int^x_{-\infty}f(t)dt< \int^x_{-\infty}\frac {f'(t)} {f(t)} f(t)dt= f(x)-\lim_{x\rightarrow-\infty}f(x)= f(x)$.
%
%This proves the lemma. 
%\end{proof}

The proof for boundedness is similar to the proof of a similar condition for BTL by~\citet{Hunter04:MM}.
\end{sketch}

\section{Asymptotic Properties of RBCML}
\label{sec:asymptotic}
Given any RUM $\mm$ and any parameter $\vec\theta$, we define $\exll(\vec\theta)=E[\cll(\vec\theta, R)]$ and let $\nabla\exll(\vec\theta)$ be the gradient of $\exll(\vec\theta)$, whose $i$th element is $\nabla_i \exll(\vec\theta)=\sum_{i'\neq i}(\frac {\bar\kappa_{ii'}w_{ii'}} {p_{ii'}(\vec\theta)}\frac {\partial p_{ii'}(\vec\theta)} {\partial\theta_i}+\frac {\bar\kappa_{i'i}w_{i'i}} {p_{i'i}(\vec\theta)}\frac {\partial p_{i'i}(\vec\theta)} {\partial\theta_i})$.
Let $H(\vec\theta, P)$ be the Hessian matrix evaluated at $\vec\theta$. And let $H_0(\vec\theta_0)$ denote the expected Hessian of $\cll(\vec\theta, P)$ at $\vec\theta_0$, where $\vec\theta_0$ is the ground truth parameter.
\begin{thm}[Consistency and asymptotic normality]\label{thm:asymptotic}
Given any RUM $\mm$, any $\vec\theta_0$ and any profile $P$ with $n$ rankings. Let $\vec\theta^*$ be the output of $\rbcml{\mg}{\mw}$. When $n\rightarrow\infty$, we have  $\vec\theta^*\xrightarrow{p}\vec\theta_0$ and

\scalebox{0.9}{
$\sqrt{n}(\vec\theta^*-\vec\theta_0)\xrightarrow{d}N(0, H^{-1}_0(\vec\theta_0)\text{Var}[\nabla \cll(\vec\theta_0, R)]H^{-1}_0(\vec\theta_0))$
} if and only if $\vec\theta_0$ is the only solution to 
 \begin{equation}\label{eqfirstorder}
 \nabla \exll(\vec\theta)=\vec 0,
 \end{equation}
 \end{thm}
\begin{proof} The ``only if" direction is straightforward. The solution to \eqref{eqfirstorder} is unique because $\cll(\vec\theta, P)$ is strictly concave. Suppose $\vec\theta_1$, other than $\vec\theta_0$, is the solution to \eqref{eqfirstorder}, then when $n\ra\infty$, $\vec\theta_1$ will be the estimate of $\rbcml{\mg}{\mw}$, which means $\rbcml{\mg}{\mw}$ is not consistent.

Now we prove the ``if" direction. First we prove consistency. {It is required by \citet{Xu2011:On-the-robustness} that for different parameters, the probabilities for any composite likelihood event are different, which is not true in our case. A simple counterexample is $\theta^{(1)}_1=1, \theta^{(2)}_1=2, \theta^{(1)}_2=\theta^{(1)}_3=\theta^{(2)}_2=\theta^{(2)}_3=0$. Then $\Pr(a_2\succ a_3|\vec\theta^{(1)})=\Pr(a_2\succ a_3|\vec\theta^{(2)})$.}
%\lirong{discuss which condition it violoates for~\citet{Xu2011:On-the-robustness}}

By the law of large numbers, we have for any $\epsilon$, $\Pr(|\cll( \vec\theta,P)-\exll(\vec\theta)|\leq \epsilon/2)\rightarrow 1$ as $n\rightarrow\infty$. This implies $\lim_{n\ra\infty}\Pr(\cll( \vec\theta^*,P)\leq \exll(\vec\theta^*)+\epsilon/2)=1$. Similarly we have $\lim_{n\ra\infty}\Pr(\exll(\vec\theta_0)\leq \cll( \vec\theta_0,P)+\epsilon/2)=1$. Since $\vec\theta^*$ maximize $\cll( \vec\theta,P)$, we have $\Pr(\cll( \vec\theta_0,P)\leq \cll( \vec\theta^*,P))=1$.
%\end{equation}
The above three equations imply that $\lim_{n\ra\infty}\Pr(\exll(\vec\theta_0)-\exll(\vec\theta^\ast)\leq \epsilon)=1$.

Let $\Theta_\epsilon$ be the subset of parameter space s.t. $\forall \vec\theta\in\Theta_\epsilon$, $\exll(\vec\theta_0)-\exll(\vec\theta)\leq \epsilon$. Because $\exll(\vec\theta)$ is strictly concave, $\Theta_\epsilon$ is compact and has a unique maximum at $\vec\theta_0$. Thus for any $\epsilon>0$, $\lim_{n\ra\infty}\Pr(\vec\theta^*\in\Theta_\epsilon)=1$. This implies consistency, i.e., $\vec\theta^*\xrightarrow{p}\vec\theta_0$.

Now we prove asymptotic normality. By mean value theorem, we have
%\begin{align*}
$0=\nabla \cll( \vec\theta^*,P)
=\nabla \cll( \vec\theta_0,P)+H(\alpha\vec\theta^*+(1-\alpha)\vec\theta_0, P)(\vec\theta^*-\vec\theta_0)$, %\end{align*}
where $0\leq \alpha\leq 1$. Therefore, we have
$\sqrt n(\vec\theta^*-\vec\theta)
=-H^{-1}(\alpha\vec\theta^*+(1-\alpha)\vec\theta_0, P)(\sqrt n\nabla \cll( \vec\theta_0,P))
$. 
Since $\nabla \cll( \vec\theta_0,P)=\frac 1 n\sum^n_{j=1}\nabla \cll(\vec\theta_0, R_j)$, by the central limit theorem, we have

$\hfill
\sqrt n\nabla \cll( \vec\theta_0,P)\xrightarrow{d} N(0, \text{Var}[\nabla \cll(\vec\theta_0, R)])
\hfill$

Because $\vec\theta^*\xrightarrow{p}\vec\theta_0$ and $H$ is continuous, we have $H(\alpha\vec\theta^*+(1-\alpha)\vec\theta_0, P)\xrightarrow{p} H(\vec\theta_0,P)$. Since $H(\vec\theta, P)=\frac 1 n\sum^n_{j=1}H(\vec\theta, R_j)$, by law of large numbers, we have $H(\vec\theta, P)\xrightarrow{p} H_0(\vec\theta_0)$. Therefore, we have
$$\sqrt n(\vec\theta^*-\vec\theta)=-H^{-1}_0(\vec\theta_0)(\sqrt n\nabla \cll( \vec\theta_0,P)),$$
which implies that $\text{Var}[\sqrt n(\vec\theta^*-\vec\theta)]=H_0^{-1}(\vec\theta_0)\text{Var}[\nabla \cll(\vec\theta_0, R)]H_0^{-1}(\vec\theta_0).$
\end{proof}

%\begin{lem} For any RUM location family, one necessary condition for maximum composite log-likelihood to be consistent is, for any $a_i$ and $a_{i'}$ where $\kappa_{ii'}, \kappa_{i'i}\neq 0$, we have $w_{ii'}=w_{i'i}$.
%\end{lem}
%\begin{proof} Suppose there exist two alternatives, w.l.o.g, $a_1$ and $a_2$, that $w_{12}\neq w_{21}$. We will construct a counterexample. Let $\theta_1=\theta_2=0$, and let $\theta_3=\cdots=\theta_m=L$ for some large $L$. Then, with probability that goes to $1$, $a_1$ and $a_2$ are ranked at the bottom two positions. We note that for any $i_1\neq i_2$, $\dfrac{\kappa_{i_1i_2}}{p_{i_1i_2}}\le 1$. Further, for any $i_1=1, 2$ and $i_2\geq 3$, we have
%% (we assume that the weights in the breaking graph are no more than one). 
%$\lim_{L\ra\infty}\frac {d p_{i_1i_2}(L)} {d L}=0$. Therefore, we have
%
%$$\frac{\kappa_{i_1i_2}w_{i_1i_2}}{p_{i_1i_2}}\frac {d p_{i_1i_2}(L)} {dL}+\frac {\kappa_{i_2i_1}w_{i_2i_1}}{p_{i_2i_1}}\frac {dp_{i_2i_1}(L)} {dL}=0$$
%
%Because the shapes of $\pi_1$ and $\pi_2$ are the same, we have $\frac{n_{12}}{p_{12}}p_{12}'(0)=\frac{n_{21}}{p_{21}}p_{21}'(0)>0$. This means that $w_{12}=w_{21}$
%% For any $i\ge 3$, we have that $n_{1i}\ra 0$ and $n_{i1}\le \sum_{k=1}^{m-2}(w_G(k,m-1)+w_G(k,m))$. We also have that $p_{i1}(L)\ra 1$
%\end{proof}
%

\section{Consistency of RBCML}

Formal proofs of theorems in this section depends on a series of lemmas, which can be found in the appendix. The full proofs can also be found in the appendix.

%Due to Theorems~\ref{thm:logconcavepl} and \ref{thm:logconcaverum}, if $\mw\otimes G(P)$ is not strongly connected, $\rbcml{\mg}{\mw}$ is non-unique or unbounded, which means \eqref{eqfirstorder} has either no solution or multiple solution. Then $\rbcml{\mg}{\mw}$ is not consistent by Theorem~\ref{thm:asymptotic}.

\begin{thm}\label{thm:pl}
$\rbcml{\mg}{\mw_\text{u}}$ is consistent for Plackett-Luce  if and only if the breaking is weighted union of position-$k$ breakings.
\end{thm}
\begin{sketch}
The ``if" direction is proved in \citep{Khetan16:Data}. We only prove the ``only if" direction by induction on $m$. When $m=2$, the only breaking is the comparison between the two alternatives. The conclusion holds. 

Suppose it holds for $m=l$, then when $m=l+1$, we first prove a lemma which says that  by restricting $\mg$ to any set of continuous positions, the theorem must hold for the subgraph. Then, we focus on $\mg_{[2, m]}$, which is the subgraph of $\mg$ on \{2,\ldots,m\}. $\mg_{[2,m]}$ must be a weighted union of position-$k$ breakings. Then we focus on $\mg_{[1, m-1]}$. The only remaining case is to prove that 
the weight on edge $\{1,m\}$ is the same as the weight on edges $\{1,i\}$ for all $i\le m-1$.

Suppose for the sake of contradiction this is not true, then we can subtract a weighted union of position-$k$ breakings from the graph, so that the remaining graph has a single edge $\{1,m\}$. We then prove that such an single-edge breaking is inconsistent by proving that~\eqref{eqfirstorder} is not satisfied, which leads to a contradiction.
\end{sketch}

\begin{figure*}[htp]
\centerline{\includegraphics[width=0.5\textwidth]{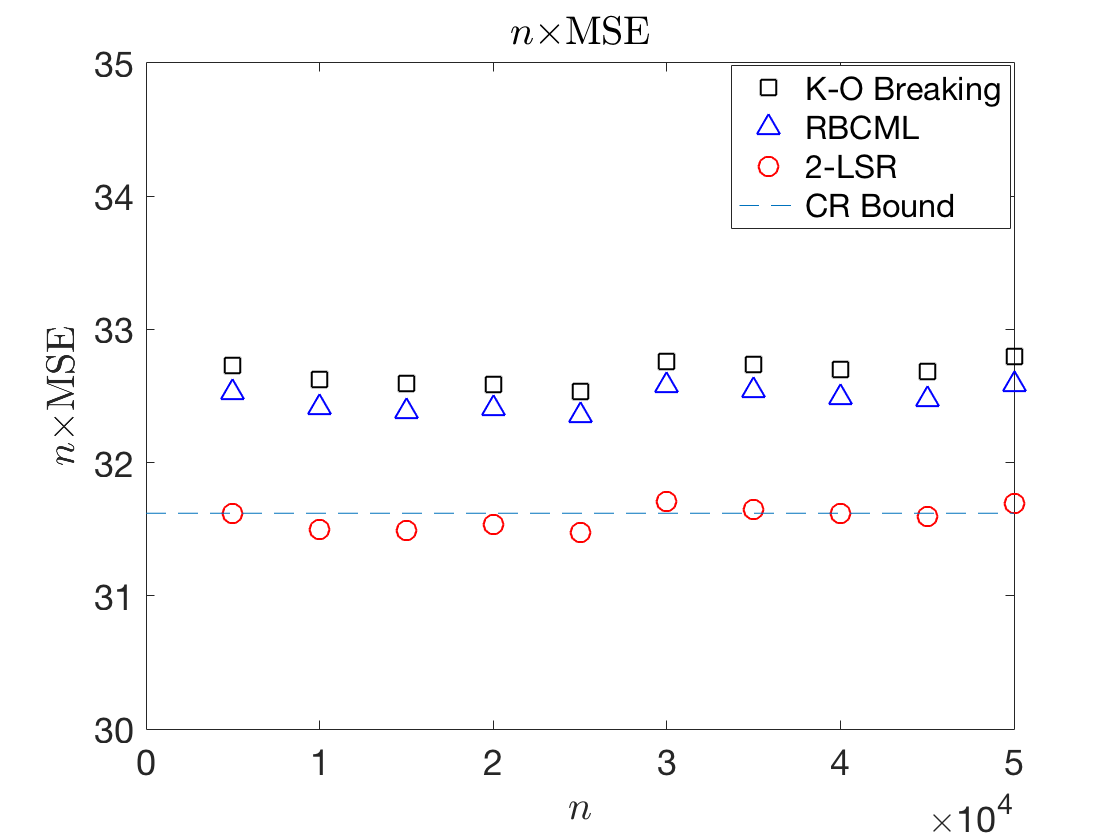}\includegraphics[width=0.5\textwidth]{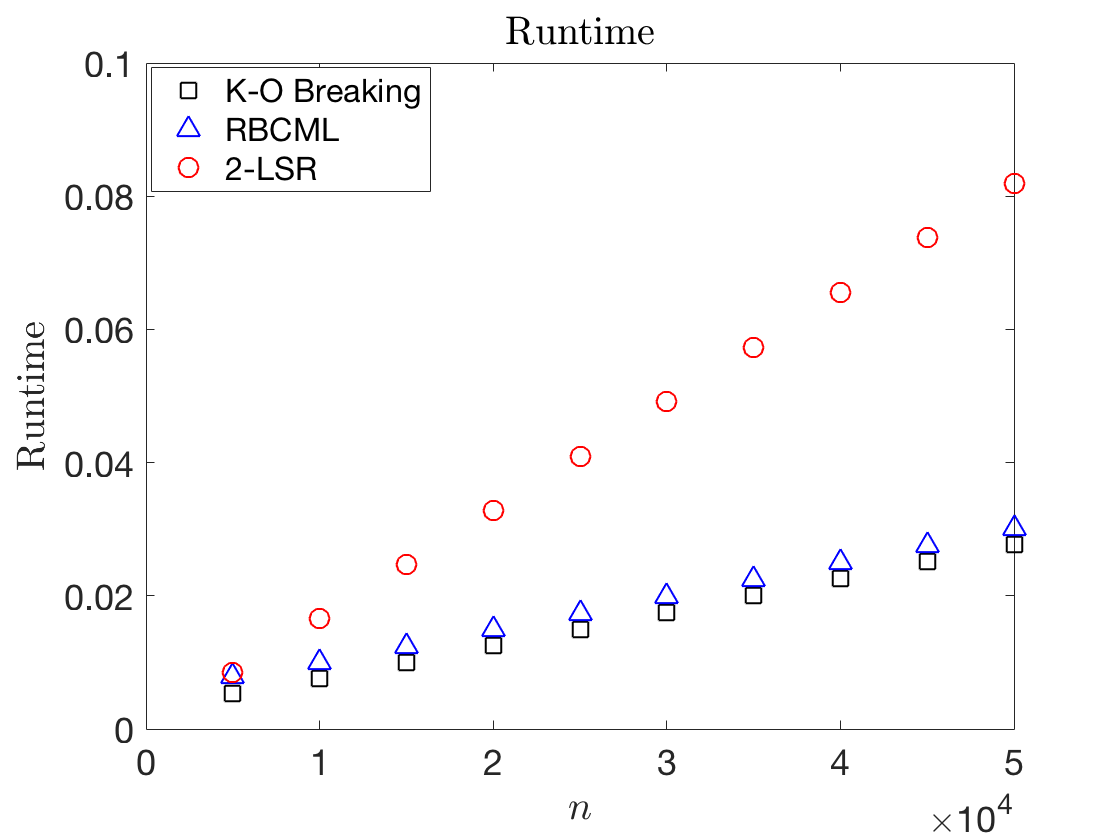}}
\caption{The $n\times$MSE and running time for the Plackett-Luce model. Values are calculated over 50000 trials. ``K-O Breaking" denotes the algorithm by~\citet{Khetan16:Data}, ``RBCML" denotes the proposed RBCML with heuristic $\mw$, ``2-LSR" denotes the 2-iteration I-LSR algorithm by~\citet{Maystre15:Fast}. ``CR Bound" line is the lower bound of $n\times$MSE for any unbiased estimator.}
\label{fig:pl}
\end{figure*}

\begin{thm}\label{thm:rum} Let $\pi_1, \pi_2, \ldots, \pi_m$ denote the utility distributions for a symmetric RUM. Suppose there exists $\pi_i$ s.t.~(1) $(\ln\pi_i(x))'$ is monotonically decreasing, and (2) $\lim_{x\ra -\infty}(\ln\pi_i(x))'\ra \infty$. Then, $\rbcml{\mg}{\mw_{\text u}}$ is consistent if and only if $\mg$ is uniform.
\end{thm}
\begin{sketch}
Define the single-edge breaking $\mg_1=\{g_{1m}=1\}$. We first prove $\rbcml{\mg_1}{\mw_{\text u}}$ is not consistent. Then we prove the theorem by induction on $m$. $m=2$ is trivial because the only breaking is uniform. For $m=3$, we first prove that the single-edge breaking $\mg_1=\{g_{13}=1\}$ is not consistent. Suppose the breaking is $\mg=\{g_{12}=x, g_{23} = y, g_{13} = z\}$. Let $\mg^*=\{g_{12}=y, g_{23} = x, g_{13} = z\}$. We prove that $\rbcml{\mg^*}{\mw_{\text u}}$ is consistent for $\mm^\ast$, which is the RUM obtained from $\mm$ by flipping the shapes of the utility distributions. Because $\mm$ is symmetric, we have $\mm^*=\mm$. Then we prove that  $\rbcml{\mg+\mg^*}{\mw_{\text u}}$ is consistent. If $x+y<2z$, We subtract $(x+y)\mg_{\text u}$ from  $\mg+\mg^*$ and get a consistent breaking $(2z-(x+y))\mg_{1}$, which is a contradiction.
%Now we prove any one-edge breaking is not consistent. For the purpose of contradiction suppose $B_{\{\{1, 2\}\}}$ is consistent. By Lemma~\ref{lem:flip} $B_{\{\{1, 2\}\}^\ast}=B_{\{\{2, 3\}\}}$ is consistent for $\mm^\ast$, which is $\mm$ due to the symmetry of utility distributions. Then we have the uniform breaking is not consistent due to Lemma~\ref{lem:plusminus}, which is a contradiction. For similar reason $B_{\{\{2, 3\}\}}$ is not consistent. 
%
%Now we prove any weighted combination of two-edge breaking is not consistent. Suppose $B_{\{x\{1,2\},y\{2,3\}\}}$ is consistent, then by Lemma~\ref{lem:flip} $B_{\{x\{1,2\},y\{2,3\}\}^\ast}=B_{\{x\{2,3\},y\{1,2\}\}}$ is also consistent. By Lemma~\ref{lem:plusminus} $B_{\{(x+y)\{1,2\},(x+y)\{2,3\}\}}$ is consistent. Because uniform breaking is consistent, we know $B_{\{(x+y)\{1, 3\}\}}$ is consistent by Lemma~\ref{lem:plusminus}, which is a contradiction. Similarly neither $B_{\{x\{1,2\},y\{1,3\}\}}$ nor $B_{\{x\{1,3\},y\{2,3\}\}}$ is consistent.
For the case where $x+y=2z$ we use the premise in the theorem statement to directly prove that the breaking is inconsistent.

%Let $\theta_1\neq\theta_2=\theta_3=0$. Further let
%\begin{align*}
%\Pr(a_1\succ a_2\succ a_3) &= \Pr(a_1\succ a_3\succ a_2) = p_1\\
%\Pr(a_2\succ a_1\succ a_3) &= \Pr(a_3\succ a_1\succ a_2) = p_2\\
%\Pr(a_2\succ a_3\succ a_1) &= \Pr(a_3\succ a_2\succ a_1) = p_3
%\end{align*}
%Then we have $\Pr(a_1\succ a_2)=2p_1+p_2$, $\Pr(a_2\succ a_1)=p_2+2p_3$
%\begin{equation}\label{sumhalf}
%p_1+p_2+p_3=\frac 1 2
%\end{equation}
%\eqref{rumdr} becomes
%\begin{align*}
%\frac {\partial \cll(\vec\theta)} {\partial\theta_1} &= \sum_{i=2, 3}(\frac {\kappa_{1i}} {p_{1i}(\vec\theta)}\frac {\partial p_{1i}(\vec\theta)} {\partial\theta_1}+\frac {\kappa_{i1}} {p_{i1}(\vec\theta)}\frac {\partial p_{i1}(\vec\theta)} {\partial\theta_1})\\
%&= 2(\frac {\kappa_{12}} {p_{12}(\vec\theta)}\frac {\partial p_{12}(\vec\theta)} {\partial\theta_1}+\frac {\kappa_{21}} {p_{21}(\vec\theta)}\frac {\partial p_{21}(\vec\theta)} {\partial\theta_1})\\
%&= 2\frac {\partial p_{12}(\vec\theta)} {\partial\theta_1}(\frac {\frac 3 2 p_1} {2p_1+p_2}-\frac {p_2+\frac 1 2 p_3} {p_2+2p_3})=0
%\end{align*}
%We have
%\begin{equation}\label{condconsist}
%\frac 3 2 p_1(p_2+2p_3)=(2p_1+p_2)(p_2+\frac 1 2 p_3)
%\end{equation}
%Combining \eqref{sumhalf} and \eqref{condconsist} we have $p_2=\sqrt{3p_3+\frac 1 {16}}-2p_3-\frac 1 4$ and $p_1=\frac 1 2 -p_2-p_3$. Let $\theta_1=2$ and standard deviations of all utility distributions be $1$, then we have $p_1=0.4329, p_2=0.0556, p_3=0.0115$, where \eqref{condconsist} does not hold. So the breaking for RUM-Gaussian is not consistent.

Suppose the theorem holds for $m=k$. When $m=k+1$, W.l.o.g. we let $\pi_2$ satisfy the conditions that $(\ln\pi_i(x))'$ is monotonically decreasing and $\lim_{x\ra -\infty}(\ln\pi_i(x))'\ra \infty$. Let $\theta_1=L$, $\theta_m=-L$, and $\theta_2=\ldots=\theta_{m-1}=0$. So when $L\ra\infty$, with probability goes to $1$, $a_1$ is ranked at the top and $a_m$ is ranked at the bottom. We then focus on  $\mg_{[2,m]}$ and $\mg_{[1, m-1]}$. By induction hypothesis, $\mg_{[2,m]}$ (respectively, $\mg_{[1, m-1]}$) is either uniform or empty. If $\mg_{[2,m]}$ is empty, then $\mg_{[1, m-1]}$ is also empty. Because $\mg$ is nonempty, we must have $\mg=C\mg_1$, where $C>0$. This is a contradiction. If $\mg_{[2,m]}$ is uniform but $\mg$ is not uniform, then the single edge breaking $\mg_1$ must be consistent, which is a contradiction.
\end{sketch}

\begin{coro} %The Gaussian distribution with any variance satisfies the condition in Lemma~\ref{lem:g210}.  Therefore, 
Theorem~\ref{thm:rum} holds for any RUM with symmetric distributions where any single distribution is Gaussian.
\end{coro}

The following two theorems give stronger characterizations by leveraging Theorems~\ref{thm:pl} and~\ref{thm:rum}.

\begin{thm}\label{thm:cmlrbconsistencypl}
 $\rbcml{\mg}{\mw}$ for Plackett-Luce is consistent if and only if $\mg$ is the weighted union of position-$k$ breakings and $\mw$ is connected and symmetric.
\end{thm}
\begin{thm}\label{thm:cmlrbconsistencyrum}
Let $\pi$ be any symmetric distribution that satisfies the condition in Theorem~\ref{thm:rum}. Then $\rbcml{\mg}{\mw}$ is consistent for RUM$(\pi)$ if and only if $\mg$ is uniform and $\mw$ is connected and symmetric.
\end{thm}
The proofs for Theorems~\ref{thm:cmlrbconsistencypl} and~\ref{thm:cmlrbconsistencyrum} are similar. The ``if" direction can be proved by verifying that the ground truth parameter is the solution to \eqref{eqfirstorder}. For the ``only if" direction, we first prove that consistency of $\rbcml{\mg}{\mw}$ implies consistency of $\rbcml{\mg}{\mw_\text{u}}$, which further implies $\mg$ is the weighted union of position-$k$ breakings for PLs (Theorem~\ref{thm:pl}) or uniform breaking for RUMs (Theorem~\ref{thm:rum}). Given this condition on $\mg$, we prove that $\mw$ must be connected and symmetric.

\section{The RBCML Framework}

The asymptotic covariance of RBCML depends on $\mg$ and $\mw$. %Thus, there exists an optimal $\mg$ and $\mw$ s.t. the asymptotic covariance is minimized. However, 
The optimal $\mg$ and $\mw$ depend on the ground truth parameter $\vec\theta_0$\footnote{\citet{Khetan16:Data} proposed a breaking $\mg$, which is not a function of $\vec\theta_0$.}, which is exactly what we want. To tackle this problem, we propose the adaptive RBCML framework, guided by our Theorems~\ref{thm:cmlrbconsistencypl} and \ref{thm:cmlrbconsistencyrum} and shown as Algorithm ~\ref{alg:rbcml}. In this algorithm, $\mg$ and $\mw$ are iteratively updated given the estimate of $\vec\theta$ from the previous iteration. 

\vspace{-3mm}
\begin{algorithm}[H]
\caption{Adaptive RBCML}
\label{alg:rbcml}
{\bf Input}: Profile $P$ of $n$ rankings, the number of iterations $T$, the heuristics of breaking $\mg(\vec\theta)$ and the weights $\mw(\vec\theta)$.\\
{\bf Output}: Estimated parameter $\vec\theta^*$.\\
{\bf Initialize} $\vec\theta^{(0)} = \vec 0$
\begin{algorithmic}[1]
\FOR{$t=1$ {\bf to} $T$}
\STATE Compute $\mg(\vec\theta^{(t-1)})$ and $\mw(\vec\theta^{(t-1)})$.
\STATE Estimate $\vec\theta^{(t)}$ using $\mg(\vec\theta^{(t-1)})$ and $\mw(\vec\theta^{(t-1)})$ by maximizing \eqref{eq:cllrum} (or \eqref{eq:cllpl} for Plackett-Luce)
\ENDFOR
\end{algorithmic}
\end{algorithm}
\vspace{-3mm}

No efficient way of computing the optimal $\mg(\vec\theta)$ and $\mw(\vec\theta)$ is known since the asymptotic covariance is generally hard to compute, where an expectation is taken over $m!$ rankings. How to efficiently compute the optimal $\mg$ and $\mw$ is a promising future direction. In the experiments of this paper, we use $\mg_\text{u}$ and $\mw_{\text{u}}$ for Gaussian RUMs since $\mg_\text{u}$ is the only consistent breaking. For the Plackett-Luce model, we use the $\mg$ proposed by \citet{Khetan16:Data} and a heuristic $\mw(\vec\theta)$ (See Section~\ref{sec:exp}).

\section{Experiments}\label{sec:exp}

\begin{figure*}[htp]
\centerline{\includegraphics[width=0.5\textwidth]{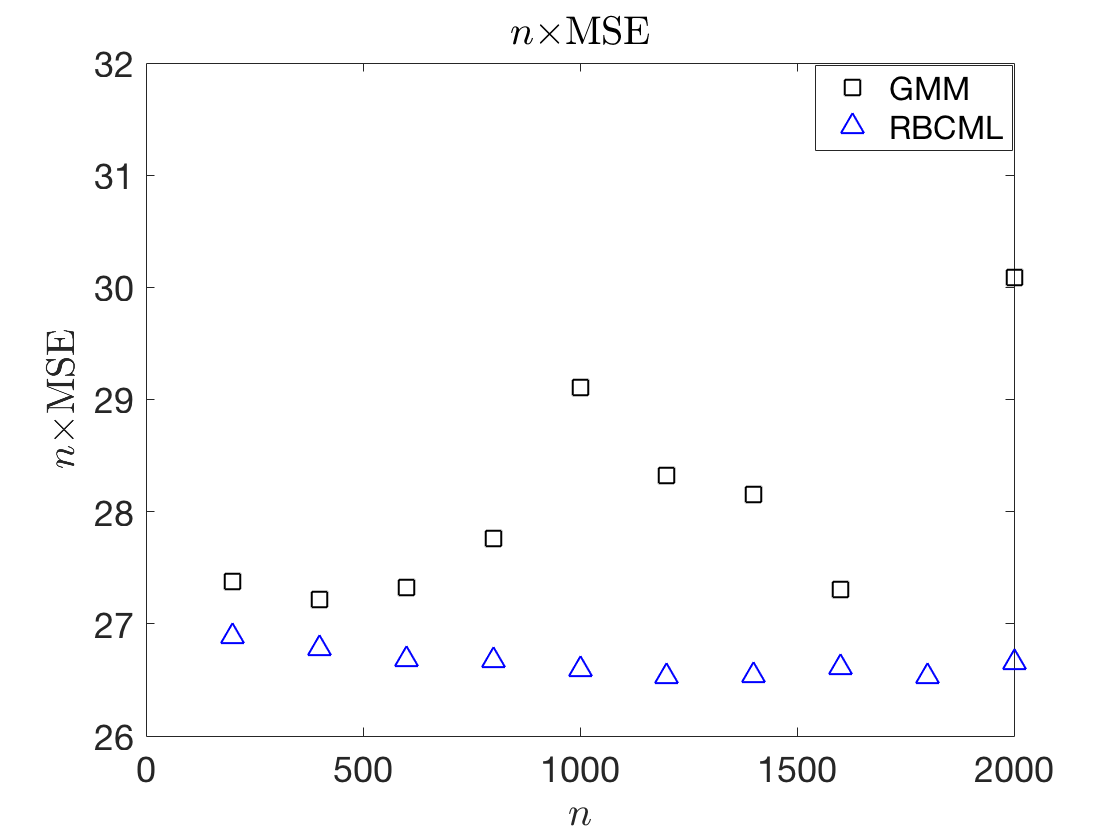}\includegraphics[width=0.5\textwidth]{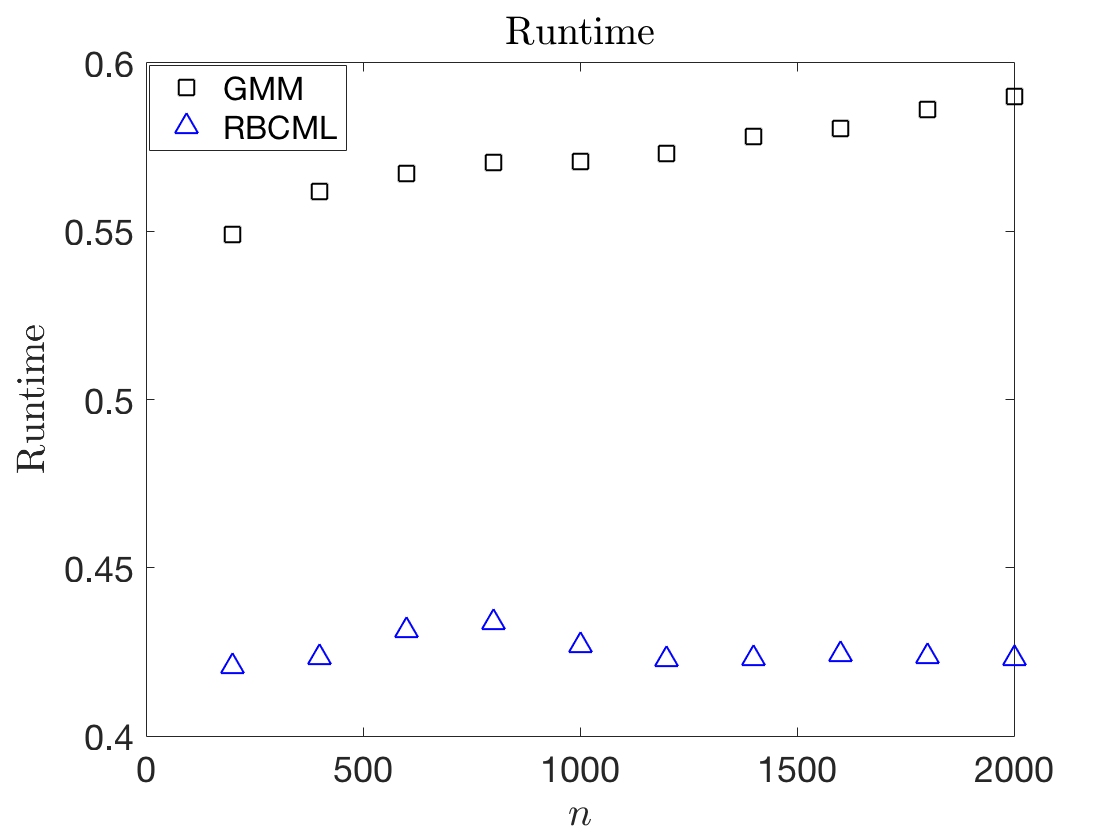}}
\caption{$n\times$MSE and runtime of GMM and RBCML for Gaussian RUMs over 10 alternatives. Values are averaged over 50000 trials.}
\label{fig:grum}
\end{figure*}

We compare RBCML with state-of-the-art algorithms for both Gaussian RUMs (GMM algorithm by~\citet{Azari14:Computing}) and the Plackett-Luce model (the I-LSR algorithm by~\citet{Maystre15:Fast} and the consistent rank-breaking algorithm by~\citet{Khetan16:Data}). In both experiments, we generate synthetic datasets of full rankings over $m=10$ alternatives. The ground truth parameter is generated uniformly at random between $0$ and $5$ and shifted s.t. $\theta_{10}=0$. For Gaussian RUMs, the utility distribution of $a_i$ is $N(\theta_i, 1)$. The results are averaged over 50000 trials. 

\noindent{\bf Metrics.} We measure statistical efficiency by $n\times\text{MSE}$, where $n$ is the number of rankings in the dataset. {\bf We use $n\times$MSE rather than the standard MSE, because it is easier to see the difference between algorithms w.r.t.~the former.}  The reason is that  $n\times$MSE approaches a positive constant as $n\rightarrow\infty$, due to asymptotic normality of RBCML.
%, MSE decreases at the rate of $1/n$ as $n$ increases Then which makes it easier to compare the performances of different algorithms. 
We use running time to measure computational efficiency of each algorithm.

{\bf Gaussian RUMs.} We use a one-step ($T=1$ in Algorithm~\ref{alg:rbcml}) RBCML$(\mg_{\text u}, \mw_{\text u})$ for Gaussian RUMs and the results are shown in Figure~\ref{fig:grum}. We use uniform breaking rather than other breakings because it is the only consistent breaking according to our theoretical results. 

We observe that our RBCML outperforms the GMM algorithm by~\citet{Azari14:Computing} w.r.t. both statistical efficiency and computational efficiency. 

{\bf The Plackett-Luce Model.} We use a two-step ($T = 2$ in Algorithm~\ref{alg:rbcml}) RBCML, where the first step is exactly the algorithm by~\citet{Khetan16:Data} (denoted by K-O Breaking). In the second step, we still use the breaking by~\citet{Khetan16:Data} but propose a heuristic $\mw(\vec\theta)$. For any pair of alternatives $a_{i_1}$ and $a_{i_2}$, we let $w_{i_1i_2}=w_{i_2i_1}=\frac 1 {|\theta_{i_1}-\theta_{i_2}|+4}$. The intuition is that we should put a higher weight on the pair of alternatives that are closer to each other. Moreover, we use the output of the first step as the starting point of the second step optimization to improve computational efficiency. 

The results are shown in Figure~\ref{fig:pl}. We use 2-LSR to denote the two-iteration I-LSR algorithms by~\citet{Maystre15:Fast}. LSR (one-iteration I-LSR) results are not shown because of the high $n\times$MSE and runtime for large $n$. The ``CR bound" line is $n$ times the trace of Cram\'er-Rao bound \cite{Cramer46:Mathematical,Rao45:Information}, which is the lower bound of the covariance matrix of any unbiased estimator. Because Cram\'er-Rao bound decreases at the rate of $1/n$, the CR bound line is horizontal. Since RBCML is not necessarily unbiased, the Cram\'er-Rao bound is not a lower bound for RBCML.

We observe that on datasets with large numbers of rankings (``$\succ$" means ``is better than"):\\
$\bullet$ Statistical efficiency: 
2-LSR $\succ$ RBCML $\succ$ K-O Breaking.
$\bullet$ Runtime:
K-O Breaking $\succ$ RBCML $\succ$ 2-LSR.

{\bf Beyond the experiments.} We have only shown the RBCML with simple $\mg$ and $\mw$. Other configurations of $\mg$ and $\mw$ can potentially have better performances or achieve other tradeoffs. Exploring RBCMLs for Gaussian RUMs, the Plackett-Luce model, as well as other RUMs is an interesting direction for future work.

\section{Summary and Future Work} We propose a flexible rank-breaking-then-composite-marginal-likelihood (RBCML) framework for learning RUMs. We characterize conditions for the objective function to be strictly log-concave, and for RBCML to be consistent and asymptotically normal. Experiments show that RBCML for Gaussian RUMs improve both statistical efficiency and computational efficiency, and the proposed RBCML for the Plackett-Luce model is competitive against state-of-the-art algorithms in that it provides a tradeoff between statistical efficiency and computational efficiency. For future work we plan to find efficient ways to compute optimal choices of $\mg$ and $\mw$, and to extend the algorithm to partial orders.

\section*{Acknowledgments}
We thank all anonymous reviewers for helpful comments
and suggestions. This work is supported by NSF \#1453542
and ONR \#N00014-17-1-2621.

\bibliography{references}
\bibliographystyle{icml2018}
%\Omit
%{
\onecolumn
\newpage
\section*{Appendix: Proofs}
%%%%%%%%%%Lemmas
%Lemma 1
\begin{lem}\label{lem:tailconc}
Let $f(x)$ be a continuously strictly log-concave differentiable probability density function with support $(-\infty, +\infty)$. $F(x)=\int^x_{-\infty}f(t)dt$ is strictly log-concave.
\end{lem}

\begin{proof} The proof is slightly modified from~\citep{Bagnoli2005:Log-Concave}. 
%Log-concave probability and its applications. Bagnoli, Mark; Bergstrom Ted.
We will prove $\frac {\partial^2\ln F(x)} {\partial x^2}=\frac {d} {dx} (\frac {f(x)} {F(x)})=\frac {f'(x)F(x)-f(x)^2} {F(x)^2}<0$. Since $F(x)>0$, we only need to prove $
f'(x)F(x)-f(x)^2<0
$.

Because $f(x)$ is strictly log-concave, we have that $\frac {d\ln f(x)} {dx} = \frac {f'(x)} {f(x)}$ 
is decreasing for any $x\in\mathbb R$. So we have $\frac {f'(x)} {f(x)}F(x)=\frac {f'(x)} {f(x)}\int^x_{-\infty}f(t)dt< \int^x_{-\infty}\frac {f'(t)} {f(t)} f(t)dt= f(x)-\lim_{x\rightarrow-\infty}f(x)= f(x)$.

This proves the lemma. 
\end{proof}

%Lemma 2
\begin{lem}\label{lem:pdr}
For any alternatives $a_i, a_{i'}$ with distributions $\pi_i, \pi_{i'}>0$ defined on $(-\infty, +\infty)$, we define $L=\theta_i-\theta_{i'}$ and let $p_{ii'}(\vec\theta)$ denote the probability of $a_i\succ a_{i'}$ given $\pi_i$ and $\pi_{i'}$. For any $\epsilon>0$, there exists $L$ s.t. $|\frac {\partial p_{ii'}(\vec\theta)} {\partial \theta_{i}}|,|\frac {\partial p_{i'i}(\vec\theta)} {\partial \theta_{i}}|\le \epsilon$.
\end{lem}
\begin{proof}
Because $p_{ii'}(\vec\theta)+p_{i'i}(\vec\theta)=1$, for any $1\leq l\leq m$, we have
\begin{equation}\label{drzero}
\frac {\partial p_{ii'}(\vec\theta)} {\partial \theta_{l}}+\frac {\partial p_{i'i}(\vec\theta)} {\partial \theta_{l}}=0
\end{equation}
%Therefore, the first order conditions an be simplified as
%\begin{equation*}
%\frac {\partial \cll(\vec\theta,P)} {\partial\theta_i} = \sum_{i'\neq i}\frac {\partial p_{ii'}(\vec\theta)} {\partial\theta_i}(\frac {\kappa_{ii'}w_{ii'}} {p_{ii'}(\vec\theta)}-\frac {\kappa_{i'i}w_{i'i}} {p_{i'i}(\vec\theta)})
%\end{equation*}

So we have $|\frac {\partial p_{ii'}(\vec\theta)} {\partial \theta_{i}}|=|\frac {\partial p_{i'i}(\vec\theta)} {\partial \theta_{i}}|$. We only need to prove $|\frac {\partial p_{ii'}} {\partial \theta_{i}}|\le\epsilon$. 

Let $\theta_{i'}=0$ and $\theta_i=L$. This is without loss of generality because $p_{ii'}(\vec\theta)$ remains the same under parameter shifts. Let $u_{i}$ and $u_{i'}$ denote the sampled utilities. We have
\begin{align*}
p_{ii'}(\vec\theta) = p_{ii'}(L)=\Pr(u_{i}>u_{i'}|\vec\theta) = \int^\infty_{-\infty}\pi_{i'}(x')\int^\infty_{x'} \pi_{i}(x-L)dxdx'=\int^\infty_{-\infty}\pi_{i'}(x')\int^\infty_{x'-L} \pi_{i}(x)dxdx'
\end{align*}
When $L$ increases, $\int^\infty_{x'-L} \pi_{i}(x)dx$ increases given any $x'$. So we have $\frac {\partial p_{ii'}(\vec\theta)} {\partial \theta_i}=\frac {d p_{ii'}(L)} {d L}\frac {\partial L} {\partial \theta_i}=\frac {d p_{ii'}(L)} {d L}>0$.
On the other hand, because $0\le p_{ii'}(L)\le 1$ we have $\int^{+\infty}_{-\infty}\frac {d p_{ii'}(L)} {d L}dL=p_{ii'}(L)|_{+\infty}-p_{ii'}(L)|_{-\infty}\leq 1$.

Therefore, for any $\epsilon$, any interval $I$ whose length is  $1/\epsilon$, we claim there exists an $L$ s.t. $\frac {\partial p_{ii'}} {\partial \theta_{i}}\le\epsilon$. The reason is as follows. Suppose for all $L\in I$, $\frac {\partial p_{ii'}} {\partial \theta_{i}}>\epsilon$ holds. Then we have
$\int^{+\infty}_{-\infty}\frac {d p_{ii'}(L)} {d L}dL > \int_I\frac {d p_{ii'}(L)} {d L}dL > \int_I\epsilon dL = \epsilon\times \frac 1 \epsilon= 1$, 
which is a contradiction.
\end{proof}

\begin{lem}\label{lem:infty}
For any alternatives $a_i, a_{i'}$ with distributions $\pi_i, \pi_{i'}>0$ defined on $(-\infty, +\infty)$. Define $L=\theta_i-\theta_{i'}$. For any $\epsilon>0$, there exists $L$ s.t.
\begin{equation*}
|\frac {\bar\kappa_{ii'}w_{ii'}} {p_{ii'}(\vec\theta)}\frac {\partial p_{ii'}(\vec\theta)} {\partial\theta_i}+\frac {\bar\kappa_{i'i}w_{i'i}} {p_{i'i}(\vec\theta)}\frac {\partial p_{i'i}(\vec\theta)} {\partial\theta_i}|\le\epsilon
\end{equation*}
\end{lem}
\begin{proof}
Let $\max\{\mg\}$ denote the maximum weight on the edges of $\mg$. Since $\frac {\bar\kappa_{ii'}} {p_{ii'}}$ is upper bounded by $\max\{\mg\}$ and $w_{ii'}$ is finite, we let $M = \max\{|\frac {\bar\kappa_{ii'}w_{ii'}} {p_{ii'}(\vec\theta)}|,|\frac {\bar\kappa_{i'i}w_{i'i}} {p_{i'i}(\vec\theta)}|\}$ and $\epsilon'=\frac \epsilon {2M}$. By Lemma~\ref{lem:pdr} there exists $L$ s.t. $|\frac {\partial p_{ii'}(\vec\theta)} {\partial\theta_i}|, |\frac {\partial p_{i'i}(\vec\theta)} {\partial\theta_i}|\leq \epsilon'$. Then we have
$
|\frac {\bar\kappa_{ii'}w_{ii'}} {p_{ii'}(\vec\theta)}\frac {\partial p_{ii'}(\vec\theta)} {\partial\theta_i}+\frac {\bar\kappa_{i'i}w_{i'i}} {p_{i'i}(\vec\theta)}\frac {\partial p_{i'i}(\vec\theta)} {\partial\theta_i}|
 \le |\frac {\bar\kappa_{ii'}w_{ii'}} {p_{ii'}(\vec\theta)}\frac {\partial p_{ii'}(\vec\theta)} {\partial\theta_i}|+|\frac {\bar\kappa_{i'i}w_{i'i}} {p_{i'i}(\vec\theta)}\frac {\partial p_{i'i}(\vec\theta)} {\partial\theta_i}|
 \le \epsilon'\times 2M=\epsilon
$
\end{proof}

\begin{lem}\label{lem:equal}
For any pair of alternatives $a_i$ and $a_{i'}$ with equal weights $w_{ii'}=w_{i'i}$, if $\theta_i=\theta_{i'}$, then we have
\begin{equation*}
\frac {\bar\kappa_{ii'}w_{ii'}} {p_{ii'}(\vec\theta)}\frac {\partial p_{ii'}(\vec\theta)} {\partial\theta_i}+\frac {\bar\kappa_{i'i}w_{i'i}} {p_{i'i}(\vec\theta)}\frac {\partial p_{i'i}(\vec\theta)} {\partial\theta_i}=0
\end{equation*}
\end{lem}
\begin{proof}
Since $\theta_i=\theta_{i'}$, we have $p_{ii'}(\vec\theta)=p_{i'i}(\vec\theta)$ and $\bar\kappa_{ii'}=\bar\kappa_{i'i}$, the lemma follows from \eqref{drzero}.
\end{proof}

\begin{lem}\label{lem:flip}
Let $\mg^*$ be the graph obtained by labeling the vertices of $\mg$ reversely, $\mm^*$ be the model obtained by flipping all of the utility distributions of $\mm$ around their means, and $\mw^*$ be the weight vector where $w^*_{ii'}=w_{i'i}$. For any RUM $\mathcal M$, if $\rbcml{\mg}{\mw}$ is consistent for $\mathcal M$, then $\rbcml{\mg^\ast}{\mw^\ast}$ is consistent for $\mathcal M^\ast$. 
\end{lem}
\begin{proof}
By Theorem~\ref{thm:asymptotic}, we only need to prove the solution to $\rbcml{\mg}{\mw}$, which is the ground truth, is the only solution to $\rbcml{\mg^\ast}{\mw^*}$. Due to strict concavity, $\rbcml{\mg^\ast}{\mw^*}$ does not have multiple solutions. So we only need to prove the solution to $\rbcml{\mg}{\mw}$ is the solution to $\rbcml{\mg^\ast}{\mw^*}$.

For any $i\in\{1, \ldots, m\}$ and any $\vec\theta$, \eqref{eqfirstorder} holds. Since $\mathcal M^\ast$ is flipped $\mathcal M$, for any ranking $R$, we have $\Pr_{\mathcal M^\ast}(R|\vec\theta)=\Pr_{\mathcal M}(rev(R)|\vec\theta)$, where $rev(R)$ is the reverse of $R$. Therefore, for any pair of alternatives $a$ and $a'$, $a\succ a'\in \mg^*(R)$ if and only if $a'\succ a\in \mg(rev(R))$.

Then for any $i\in\{1, \ldots, m\}$, we have
$$
\nabla_i\text{ELL}_{\mm^*}(\vec\theta) = \sum_{i'\neq i}(\frac {\bar\kappa_{ii'}w^*_{ii'}} {p^\ast_{ii'}(\vec\theta)}\frac {\partial p^\ast_{ii'}(\vec\theta)} {\partial\theta_i}+\frac {\bar\kappa_{i'i}w^*_{i'i}} {p^\ast_{i'i}(\vec\theta)}\frac {\partial p^\ast_{i'i}(\vec\theta)} {\partial\theta_i})
= \sum_{i'\neq i}(\frac {\bar\kappa_{i'i}w_{i'i}} {p_{i'i}(\vec\theta)}\frac {\partial p_{i'i}(\vec\theta)} {\partial(\theta_i)}+\frac {\bar\kappa_{ii'}w_{ii'}} {p_{ii'}(\vec\theta)}\frac {\partial p_{ii'}(\vec\theta)} {\partial(\theta_i)})=0.
$$
This finishes the proof of the lemma.
\end{proof}

\begin{lem}\label{lem:subgraph}
Let $\mg_{[k_1, k_2]}$ denote the subgraph $\mg$ restricted to nodes between $k_1$ and $k_2$ (inclusive). For any RUM $\mathcal M$, if $\rbcml{\mg}{\mw_\text{u}}$ is consistent, then for any $1\leq k_1<k_2\leq m$, $\rbcml{\mg_{[k_1, k_2]}}{\mw_\text{u}}$ is either empty or consistent for $k_2-k_1+1$ alternatives.
\end{lem}
\begin{proof}
We prove that if $\rbcml{\mg_{[k_1, k_2]}}{\mw_\text{u}}$ is not consistent then $\rbcml{\mg}{\mw_\text{u}}$ is not consistent. Suppose $\rbcml{\mg_{[k_1,k_2]}}{\mw_\text{u}}$ is not consistent. For convenience we keep the index of $\mg$ in $\mg_{[k_1, k_2]}$ and let $\mathcal M'$ denote the model with the $k_2-k_1+1$ alternatives. Then there exists $\theta_i$ where $k_1\leq i\leq k_2$ s.t.
$$
|\nabla_i\text{ELL}_{\mm'}(\vec\theta)| = |\sum_{k_1\leq i'\leq k_2, i'\neq i}(\frac {\bar\kappa_{ii'}w_{ii'}} {p_{ii'}(\vec\theta)}\frac {\partial p_{ii'}(\vec\theta)} {\partial\theta_i}+\frac {\bar\kappa_{i'i}w_{i'i}} {p_{i'i}(\vec\theta)}\frac {\partial p_{i'i}(\vec\theta)} {\partial\theta_i})|
=C>0
$$

We now construct other elements in $\vec\theta$ to show that $\rbcml{\mg}{\mw_{\text{u}}}$ is not consistent. We let $\theta_1=\ldots =\theta_{k_1-1}=L$ and $\theta_{k_2}+1=\ldots=\theta_r = -L$. Then when $L\ra\infty$, with probability that goes to 1, $a_1, \ldots, a_{k_1-1}$ are ranked in the top $k_1-1$ positions and $a_{k_2+1},\ldots,a_m$ are ranked in the bottom $m-k_2$ positions. 

By Lemma~\ref{lem:infty} for any $k_1\le i\le k_2$ and $i'<k_1$ (or  $i'>k_2$) there exists $L$ s.t.
$
|\frac {\bar\kappa_{ii'}w_{ii'}} {p_{ii'}(\vec\theta)}\frac {\partial p_{ii'}(\vec\theta)} {\partial\theta_i}+\frac {\bar\kappa_{i'i}w_{i'i}} {p_{i'i}(\vec\theta)}\frac {\partial p_{i'i}(\vec\theta)} {\partial\theta_i})|\le\frac C m
$. Then we have
$
|\nabla_i\exll(\vec\theta)|
%=|\sum_{i'\neq i}(\frac {\bar\kappa_{ii'}w_{ii'}} {p_{ii'}(\vec\theta)}\frac {\partial p_{ii'}(\vec\theta)} {\partial\theta_i}+\frac {\bar\kappa_{i'i}w_{i'i}} {p_{i'i}(\vec\theta)}\frac {\partial p_{i'i}(\vec\theta)} {\partial\theta_i})|\\
\ge |\nabla_i\text{ELL}_{\mm'}(\vec\theta)|-(m-(k_2-k_1+1))\frac {C} m=\frac {(k_2-k_1+1)C} m >0
$. So we have $
\nabla_i\exll(\vec\theta)\neq 0$. $\rbcml{\mg}{\mw_\text{u}}$ is thus not consistent. 
\end{proof}

\begin{lem}\label{pl1m}
For any $m\geq 3$, $\rbcml{\mg}{\mw_{\text u}}$ for the Plackett-Luce model is not consistent if $\mg=\{g_{1m}=C\}$, where $C>0$ is a constant.
\end{lem}
\begin{proof}
It suffices to prove $\rbcml{\mg}{\mw_{\text u}}$ for the Plackett-Luce model is not consistent if $\mg=\{g_{1m}=1\}$. We prove this lemma by constructing a counter-example. Let $\theta_1=x$ and $\theta_2=\ldots=\theta_m=0$. 
For any ranking $R_1$ with alternative $a_1$ at top, the probability is
$
\Pr(R_1|\vec\theta) = \frac 1 {(m-1)!}\frac {e^x} {e^x+(m-1)}
$. 
For any ranking $R_2$ with $a_1$ at bottom, the probability is
$
\Pr(R_2|\vec\theta) = \frac {1} {\prod^{m-1}_{k=1}(e^x+k)}
$.
For any $a_i$ where $2\leq i\leq m$, we have $\bar\kappa_{1i}=(m-1)!\Pr(R_1|\vec\theta)$ and $\bar\kappa_{i1}=(m-1)!\Pr(R_2|\vec\theta)$. Therefore, we have
$
\nabla_i\text{ELL}_{\text{PL}}(\vec\theta)=\sum_{i'\neq i}(\bar\kappa_{ii'}-(\bar\kappa_{ii'}+\bar\kappa_{i'i})\frac {1} {e^x+1})=(m-1)(\frac {e^{2x}} {e^x+m-1}-\frac {(m-1)!} {\prod^{m-1}_{k=1}(e^x+k)})
$. 
Let $x = \ln 2$, then we have $\nabla_i\text{ELL}_{\text{PL}}(\vec\theta)=\frac {4m-2} {m(m+1)}\neq 0$. This proves the lemma.
\end{proof}

\begin{lem}\label{lem:rum1m}
For any $m\geq 3$, $\rbcml{\mg}{\mw_\text{u}}$ for any RUM location family with the same symmetric pdf is not consistent if $\mg=\{g_{1m}=C\}$ where $C>0$ is a constant. 
\end{lem}
\begin{proof}
Let $\pi$ denote the PDF of the utility distribution for all alternatives with mean $0$. That is, for any $i\le m$ and any $x\in \mathbb R$, we have $\pi_i(x)=\pi(x-\theta_i)$. Let $B>0$ be an arbitrary number so that $1-\epsilon>\int_{-B}^B\pi(x)dx>\epsilon$. Let $L$ be a large number that will be specified later. 

We first prove the lemma for $m=3$. Let $\theta_1=L$ and $\theta_2=\theta_3=0$. Since $\theta_2=\theta_3$, we have $\frac {\bar\kappa_{12}} {p_{12}(\vec\theta)}\frac {\partial p_{12}(\vec\theta)} {\partial\theta_1}+\frac {\bar\kappa_{21}} {p_{21}(\vec\theta)}\frac {\partial p_{21}(\vec\theta)} {\partial\theta_1}=\frac {\bar\kappa_{13}} {p_{13}(\vec\theta)}\frac {\partial p_{13}(\vec\theta)} {\partial\theta_1}+\frac {\bar\kappa_{31}} {p_{31}(\vec\theta)}\frac {\partial p_{31}(\vec\theta)} {\partial\theta_1}$. Due to \eqref{drzero}, it suffices to prove $
\frac {\bar\kappa_{12}} {p_{12}(\vec\theta)} \neq \frac {\bar\kappa_{21}} {p_{21}(\vec\theta)}
$, which is equivalent to
$
\frac{\Pr(a_1\text{ top and } a_2\text{ bottom})}{\Pr(a_1\succ a_2)}\ne \frac{\Pr(a_2\text{ top and } a_1\text{ bottom})}{\Pr(a_2\succ a_1)}
$. That is
%$$
%\frac{\Pr(a_1\succ a_3\succ a_2)}{\Pr(a_1\succ a_3\succ a_2)+\Pr(a_3\succ a_1\succ a_2)+\Pr(a_1\succ a_2\succ a_3)}\ne \frac{\Pr(a_2\succ a_3\succ a_1)}{\Pr(a_2\succ a_3\succ a_1)+\Pr(a_3\succ a_2\succ a_1)+\Pr(a_2\succ a_1\succ a_3)}
%$$
$
\frac {p_{132}} {p_{312}+p_{132}+p_{123}}\neq \frac {p_{231}} {p_{321}+p_{231}+p_{213}}
$, 
where $p_{123}$ is the short form of $\Pr(a_1\succ a_2\succ a_3)$.  Because $p_{123}=p_{132}$ and $p_{231}=p_{321}$, we only need to prove
$
\frac {p_{132}} {p_{312}}\neq \frac {p_{231}} {p_{213}}
$. This is obvious because $p_{312}=p_{213}$ but $p_{132}\neq p_{231}$.

We now prove the lemma for any $m\ge 4$. Let $\theta_1=\theta_2=L$ and $\theta_3=\ldots=\theta_m=0$. By Lemma~\ref{lem:equal} we have $\frac {\bar\kappa_{12}} {p_{12}(\vec\theta)}\frac {\partial p_{12}(\vec\theta)} {\partial\theta_1}+\frac {\bar\kappa_{21}} {p_{21}(\vec\theta)}\frac {\partial p_{21}(\vec\theta)} {\partial\theta_1}=0$. For all $3\le i\le m$, we have $\frac {\bar\kappa_{1i}} {p_{1i}(\vec\theta)}\frac {\partial p_{1i}(\vec\theta)} {\partial\theta_1}+\frac {\bar\kappa_{i1}} {p_{i1}(\vec\theta)}\frac {\partial p_{i1}(\vec\theta)} {\partial\theta_1} = \frac {\bar\kappa_{1m}} {p_{1m}(\vec\theta)}\frac {\partial p_{1m}(\vec\theta)} {\partial\theta_1}+\frac {\bar\kappa_{m1}} {p_{m1}(\vec\theta)}\frac {\partial p_{m1}(\vec\theta)} {\partial\theta_1}$. 
So we have
$
\nabla_i\exll(\vec\theta) = (m-2)(\frac {\bar\kappa_{1m}} {p_{1m}(\vec\theta)}\frac {\partial p_{1m}(\vec\theta)} {\partial\theta_1}+\frac {\bar\kappa_{m1}} {p_{m1}(\vec\theta)}\frac {\partial p_{m1}(\vec\theta)} {\partial\theta_1})
$. It suffices to prove
$
\frac {\bar\kappa_{1m}} {p_{1m}(\vec\theta)} \neq \frac {\bar\kappa_{m1}} {p_{m1}(\vec\theta)}
$, which is
\begin{equation}\label{eq:ratioineqx}
\frac{\Pr(a_1\text{ top and } a_m\text{ bottom})}{\Pr(a_1\succ a_m)}\ne \frac{\Pr(a_m\text{ top and } a_1\text{ bottom})}{\Pr(a_m\succ a_1)}
\end{equation}
Because $L$ is large, $\Pr(a_1 \text{ top or } a_2 \text{ top})\approx 1$. Because $\pi_i$'s have the same shape, we have that 
\begin{align*}\Pr(a_1\text{ top and } a_m\text{ bottom})
\approx \Pr(a_1\succ a_2\text{ and }a_m \text{ is ranked lower than }a_3,\ldots, a_{m-1})
\end{align*}

Therefore, the LHS of (\ref{eq:ratioineqx}) is $\frac{1}{2(m-2)}$ as $L\ra\infty$. We will show that the RHS of  (\ref{eq:ratioineqx}) is converges to $0$  as $L\ra\infty$. We define a partition of $\{(u_1,u_m):u_1<u_m\}=S_1\cup S_2$ as follows.
\begin{itemize}
\item $S_1=\{(u_1,u_m):u_1<B\text{ and }u_m>L-B\}$,
%\item $S_2=\{(u_1,u_m):u_1<B<u_m<L-B\}$,
\item $S_2=$ others.
\end{itemize}
We further define the following two functions $\pi$ and $\pi^*$ for $u_1<u_m$.
$$\pi(u_1,u_m)=\pi_1(u_1)\times \pi_m(u_m)$$
$$\pi^*(u_1,u_m)=\pi_1(u_m)\times \pi_m(u_m)\times\prod_{i=2}^{m-1} \int_{u_1}^{u_m}\pi_i(u_i) du_i$$

It follows that
\begin{align*}
\frac{\Pr(a_m\text{ top and } a_1\text{ bottom})}{\Pr(a_m\succ a_1)}
=\frac{\int_{S_1}\pi^*(u_1,u_m) +\int_{S_2}\pi^*(u_1,u_m)}{\int_{S_1}\pi(u_1,u_m) +\int_{S_2}\pi(u_1,u_m)}
\end{align*}
\begin{claim}\label{claim:ratios1s2}$\lim_{L\ra\infty}\dfrac{\int_{S_1}\pi(u_1,u_m)}{\int_{S_2}\pi(u_1,u_m)}=0$.
\end{claim}
\begin{proof}
Let $S=\{(u_1,u_m):u_1<B<u_m<L-B\}$.  We have $\dfrac{\int_{S_1}\pi(u_1,u_m)}{\int_{S}\pi(u_1,u_m)}=\dfrac{\int_{L-B}^\infty\pi_m(u_m) d u_m}{\int_{B}^{L-B}\pi_m(u_m) d u_m}$, which converges to $0$.  The claim follows after observing that $S\subseteq S_2$.
\end{proof}

\begin{claim}\label{claim:ratios2}$\lim_{\epsilon\ra0}\dfrac{\int_{S_2}\pi^*(u_1,u_m)}{\int_{S_2}\pi(u_1,u_m)}=0$.
\end{claim}
\begin{proof}
For any $(u_1,u_m)\in S_2$, either $u_1>B$ or $u_m<L-B$. If $u_1>B$, then 
\begin{align*}
\prod_{i=2}^{m-1}\int_{u_1}^{u_m}\pi_i(u_i) du_i\le \int_{u_1}^{u_m}\pi_{m-1}(u_{m-1}) du_{m-1}
\le \int_{B}^{\infty}\pi_{m-1}(u_{m-1}) du_{m-1}\le \epsilon
\end{align*}

If $u_m<L-B$, then we have
$
\prod_{i=2}^{m-1}\int_{u_1}^{u_m}\pi_i(u_i) du_i\le \int_{u_1}^{u_m}\pi_{2}(u_{2}) du_{2}
\le \int_{-\infty}^{L-B}{\pi_{2}(u_{2}) du_{2}}\le \epsilon
$
Therefore, for any $(u_1,u_m)\in S_2$, $\dfrac{\pi^*(u_1,u_m)}{\pi(u_1,u_m)}\le \epsilon$. This proves the claim.
\end{proof}

We are now ready to prove the lemma.
\begin{align*}
&\frac{\Pr(a_m\text{ top and } a_1\text{ bottom})}{\Pr(a_m\succ a_1)}=\frac{\int_{S_1}\pi^*(u_1,u_m) +\int_{S_2}\pi^*(u_1,u_m)}{\int_{S_1}\pi(u_1,u_m) +\int_{S_2}\pi(u_1,u_m)}\\
{\le}& \frac{\int_{S_1}\pi(u_1,u_m) +\int_{S_2}\pi^*(u_1,u_m)}{\int_{S_1}\pi(u_1,u_m) +\int_{S_2}\pi(u_1,u_m)}=\frac{\frac{\int_{S_1}\pi(u_1,u_m)}{\int_{S_2}\pi(u_1,u_m)} +\frac{\int_{S_2}\pi^*(u_1,u_m)}{\int_{S_2}\pi(u_1,u_m)}}{\frac{\int_{S_1}\pi(u_1,u_m)}{\int_{S_2}\pi(u_1,u_m)} +1}
\end{align*}

Therefore, by combining Claim~\ref{claim:ratios1s2} and Claim~\ref{claim:ratios2}, we have 

$$\lim_{L\ra\infty,\epsilon\ra0}\frac{\Pr(a_m\text{ top and } a_1\text{ bottom})}{\Pr(a_m\succ a_1)}=0$$
Therefore, there exist $L$ and $\epsilon$ so that $\rbcml{\mg}{\mw_\text{u}}$ is inconsistent.
\end{proof}

Let $\mg_1$ and $\mg_2$ be a pair of weighted breakings. Define $\mg_1+\mg_2$ to be a breaking with weights being the sum of weights of corresponding edges in $\mg_1$ and $\mg_2$. Note that no edge between two vertices is equivalent to an edge with zero weight between the two vertices. If weights of all edges of $\mg_1$ are no less than those in $\mg_2$ (denoted as $\mg_1\geq \mg_2$), we define $\mg_1-\mg_2$ to be a breaking whose weight on each edge is the difference of the corresponding edge in $\mg_1$ and $\mg_2$ s.t. weights on all edges are nonnegative. 

\begin{lem}\label{lem:plusminus}
$\mg_1$ and $\mg_2$ are weighted breakings.
\begin{itemize}
\item If $\rbcml{\mg_1}{\mw_\text{u}}$ and $\rbcml{\mg_2}{\mw_\text{u}}$ are both consistent, then $\rbcml{\mg_1+\mg_2}{\mw_\text{u}}$ is also consistent. Further, if $\mg_1\geq \mg_2$, then $\rbcml{\mg_1-\mg_2}{\mw_\text{u}}$ is consistent.
\item If $\rbcml{\mg_1}{\mw_\text{u}}$ is consistent but $\rbcml{\mg_2}{\mw_\text{u}}$ is not consistent, then $\rbcml{\mg_1+\mg_2}{\mw_\text{u}}$ is not consistent. Further, if $\mg_1\geq \mg_2$, then $\rbcml{\mg_1-\mg_2}{\mw_\text{u}}$ is not consistent.
\end{itemize}
\end{lem}

\begin{proof} For any breaking $\mg$, let $\exll^\mg(\vec\theta)$  denote the expected log-marginal likelihood function under $\rbcml{\mg}{\mw_\text{u}}$.

{\bf Case 1.} Because $\rbcml{\mg_1}{\mw_\text{u}}$ and $\rbcml{\mg_2}{\mw_\text{u}}$ are both consistent, for any $1\leq i\leq m$, we have
\begin{align*}
\nabla_i\exll^{\mg_1}(\vec\theta)=&\sum_{i'\neq i}(\frac {\bar\kappa^{\mg_1}_{ii'}w_{ii'}} {p_{ii'}(\vec\theta)}\frac {\partial p_{ii'}(\vec\theta)} {\partial\theta_i}+\frac {\bar\kappa^{\mg_1}_{i'i}w_{i'i}} {p_{i'i}(\vec\theta)}\frac {\partial p_{i'i}(\vec\theta)} {\partial\theta_i})=0\\
\nabla_i\exll^{\mg_2}(\vec\theta)=&\sum_{i'\neq i}(\frac {\bar\kappa^{\mg_2}_{ii'}w_{ii'}} {p_{ii'}(\vec\theta)}\frac {\partial p_{ii'}(\vec\theta)} {\partial\theta_i}+\frac {\bar\kappa^{\mg_2}_{i'i}w_{i'i}} {p_{i'i}(\vec\theta)}\frac {\partial p_{i'i}(\vec\theta)} {\partial\theta_i})=0
\end{align*}
It follows that
\begin{align*}
\nabla_i\exll^{\mg_1+\mg_2}(\vec\theta) &= \nabla_i\exll^{\mg_1}(\vec\theta)+\nabla_i\exll^{\mg_2}(\vec\theta)=0\\
\nabla_i\exll^{\mg_1-\mg_2}(\vec\theta) &= \nabla_i\exll^{\mg_1}(\vec\theta)-\nabla_i\exll^{\mg_2}(\vec\theta)=0
\end{align*}
{\bf Case 2.} Because $\rbcml{\mg_1}{\mw_\text{u}}$ is consistent and $\rbcml{\mg_2}{\mw_\text{u}}$ is not consistent, there exists $1\leq i\leq m$ s.t. 
\begin{align*}
\nabla_i\exll^{\mg_1}(\vec\theta)=&\sum_{i'\neq i}(\frac {\bar\kappa^{\mg_1}_{ii'}w_{ii'}} {p_{ii'}(\vec\theta)}\frac {\partial p_{ii'}(\vec\theta)} {\partial\theta_i}+\frac {\bar\kappa^{\mg_1}_{i'i}w_{i'i}} {p_{i'i}(\vec\theta)}\frac {\partial p_{i'i}(\vec\theta)} {\partial\theta_i})=0\\
\nabla_i\exll^{\mg_2}(\vec\theta)=&\sum_{i'\neq i}(\frac {\bar\kappa^{\mg_2}_{ii'}w_{ii'}} {p_{ii'}(\vec\theta)}\frac {\partial p_{ii'}(\vec\theta)} {\partial\theta_i}+\frac {\bar\kappa^{\mg_2}_{i'i}w_{i'i}} {p_{i'i}(\vec\theta)}\frac {\partial p_{i'i}(\vec\theta)} {\partial\theta_i})\neq 0
\end{align*}
It follows that
\begin{align*}
\nabla_i\exll^{\mg_1+\mg_2}(\vec\theta) &= \nabla_i\exll^{\mg_1}(\vec\theta)+\nabla_i\exll^{\mg_2}(\vec\theta)\neq 0\\
\nabla_i\exll^{\mg_1-\mg_2}(\vec\theta) &= \nabla_i\exll^{\mg_1}(\vec\theta)-\nabla_i\exll^{\mg_2}(\vec\theta)\neq 0
\end{align*}
which implies inconsistency.
\end{proof}

\begin{lem}\label{lem:g210} Let $m=3$ and let $\rum(\pi_1,\pi_2,\pi_3)$ be an RUM with symmetric distributions, where for at least one $\pi_i$ we have $(\ln \pi_i)'=\frac {\pi_i'(x)} {\pi_i(x)}$ is monotonically decreasing and $\lim_{x\ra -\infty}\frac {\pi_i'(x)} {\pi_i(x)}\ra \infty$, then $\rbcml{\mg_{\{2\times\{1,2\},\{1,3\}\}}}{\mw_{\text{u}}}$ is not consistent for $\rum(\pi_1,\pi_2,\pi_3)$.
\end{lem}
\begin{proof}
Let $\mg_{210}$ denote $\mg_{\{2\times\{1,2\},\{1,3\}\}}$. W.l.o.g.~suppose $\lim_{x\ra -\infty}(\pi_1'(x))\ra \infty$. Let $\theta_1>0$ and $\theta_2=\theta_3=0$. We will prove that when $\theta_1$ is sufficiently large, Equation~\eqref{eqfirstorder} does not hold. Let
\begin{align*}
\Pr(a_1\succ a_2\succ a_3) &= \Pr(a_1\succ a_3\succ a_2) = p_1\\
\Pr(a_2\succ a_1\succ a_3) &= \Pr(a_3\succ a_1\succ a_2) = p_2\\
\Pr(a_2\succ a_3\succ a_1) &= \Pr(a_3\succ a_2\succ a_1) = p_3
\end{align*}
We have $p_1+p_2+p_3=\frac 1 2$ and $\Pr(a_1\succ a_2)=2p_1+p_2$, $\Pr(a_2\succ a_1)=p_2+2p_3$. Given $\mg_{210}$, $\bar\kappa_{12}=3p_1$ and $\bar\kappa_{21}=2p_2+p_3$. Therefore, Equation \eqref{eqfirstorder} becomes
\begin{align*}
\nabla_1\exll(\vec\theta) &= \sum_{i=2, 3}(\frac {\bar\kappa_{1i}} {p_{1i}(\vec\theta)}\frac {\partial p_{1i}(\vec\theta)} {\partial\theta_1}+\frac {\bar\kappa_{i1}} {p_{i1}(\vec\theta)}\frac {\partial p_{i1}(\vec\theta)} {\partial\theta_1})= 2(\frac {\bar\kappa_{12}} {p_{12}(\vec\theta)}\frac {\partial p_{12}(\vec\theta)} {\partial\theta_1}+\frac {\bar\kappa_{21}} {p_{21}(\vec\theta)}\frac {\partial p_{21}(\vec\theta)} {\partial\theta_1})\\
&= 2\frac {\partial p_{12}(\vec\theta)} {\partial\theta_1}(\frac {3p_1} {2p_1+p_2}-\frac {2p_2+p_3} {p_2+2p_3})=0
\end{align*}
Therefore, the following equation holds for all cases with $\theta_2=\theta_3=0$ and $\theta_1>0$.
\begin{equation}\label{eq:condconsistx}
\frac {3p_1} {2p_1+p_2}=\frac {2p_2+p_3} {p_2+2p_3}
\end{equation}
As $\theta_1\ra\infty$, $p_1\ra 0.5$ and $p_2,p_3$ goes to $0$. Equation~\eqref{eq:condconsistx} becomes $\frac {2p_2+p_3} {p_2+2p_3}=\frac 3 2$. It follows that $\lim_{\theta_1\ra\infty}\frac{p_2}{p_3}=4$. We next prove that $\lim_{\theta_1\ra\infty}\frac{p_2}{p_3}=\infty$, which will lead to a contradiction. For $i=2,3$, we let $\cdf_i$ denote the CDF of $\pi_i$.
%\begin{align*}
%\frac{p_2}{p_3}&=\frac{\int_{-\infty}^{\infty}\cdf_2(U_1) \pi_1(U_1-\theta_1)(1-\cdf_3(U_1)) d U_1+\int_{-\infty}^{\infty}\cdf_3(U_1) \pi_1(U_1-\theta_1)(1-\cdf_2(U_1)) d U_1}{\int_{-\infty}^{\infty} \pi_1(U_1-\theta_1)(1-\cdf_2(U_1))(1-\cdf_3(U_1)) d U_1}
%\end{align*}
By symmetry, it suffices to prove that 
$\lim_{\theta_1\ra\infty}\frac{\int_{-\infty}^{\infty}\pi_1(U_1-\theta_1)\cdf_2(U_1) (1-\cdf_3(U_1)) d U_1}{\int_{-\infty}^{\infty} \pi_1(U_1-\theta_1)(1-\cdf_2(U_1))(1-\cdf_3(U_1)) d U_1}=\infty$.

The idea is to choose $B$ and $\theta_1$ so that $U_1<B$ in the integration of both numerator and denominator can be ignored, and the ratio for the remainders of numeration and denominator can be arbitrarily large. More precisely, for any $K>0$, let $B>0$ denote any number such that $\dfrac{\cdf_2( B+1)}{1-\cdf_2( B+1)}>K+1$.  Let $\theta_1$ be any number such that 
\begin{align*}
{(\ln\pi_1)'(B+1-\theta_1)} >\ln (K\frac{\int_{-\infty}^{B}(1-\cdf_2(U_1)) (1-\cdf_3(U_1)) d U_1}{\int_{B+1}^{3B+1}(1-\cdf_2(U_1)) (1-\cdf_3(U_1)) d U_1})\end{align*}
Such a $\theta$ exists because $\lim_{x\ra -\infty}\frac {\pi_i'(x)} {\pi_i(x)}\ra \infty$. Because $\pi_1(x)$ is monotonically increasing  for all $x<0$, we have
\begin{align*}
&\int_{B}^{\infty}\pi_1(U_1-\theta_1)(1-\cdf_2(U_1)) (1-\cdf_3(U_1)) d U_1\\
>&\int_{B+1}^{3B+1}\pi_1(U_1-\theta_1)(1-\cdf_2(U_1)) (1-\cdf_3(U_1)) d U_1\\\
>&\pi_1(B+1-\theta_1)\times\int_{B+1}^{3B+1}(1-\cdf_2(U_1)) (1-\cdf_3(U_1)) d U_1\\\
>&e^{(\ln \pi_1)'(B+1-\theta_1)}\pi_1(B-\theta_1)\times\int_{B+1}^{3B+1}(1-\cdf_2(U_1)) (1-\cdf_3(U_1)) d U_1\\
>&K\pi_1(B-\theta_1)\int_{-\infty}^{B}(1-\cdf_2(U_1)) (1-\cdf_3(U_1)) d U_1\\
>&K\int_{-\infty}^{B}\pi_1(U_1-\theta_1)(1-\cdf_2(U_1)) (1-\cdf_3(U_1)) d U_1\\
\end{align*}
Therefore, we have 
\begin{align*}
&\frac{\int_{-\infty}^{\infty}\pi_1(U_1-\theta_1)\cdf_2(U_1) (1-\cdf_3(U_1)) d U_1}{\int_{-\infty}^{\infty} \pi_1(U_1-\theta_1)(1-\cdf_2(U_1))(1-\cdf_3(U_1)) d U_1}\\
>&\frac{\int_{B+1}^{\infty}\pi_1(U_1-\theta_1)\cdf_2(U_1) (1-\cdf_3(U_1)) d U_1}{(1+\frac{1}{K})\int_{B+1}^{\infty} \pi_1(U_1-\theta_1)(1-\cdf_2(U_1))(1-\cdf_3(U_1)) d U_1}\\
>&\frac{\cdf_2(B+1) (1-\cdf_3(B+1))}{(1+\frac{1}{K})(1-\cdf_2(B+1))(1-\cdf_3(B+1))}
>K
\end{align*}
Therefore, it is impossible that Equation~\eqref{eq:condconsistx} holds for all $\theta_1$, which proves the lemma.
\end{proof}

\begin{lem}\label{lem:consistency}
1. For any location family RUM$(\pi_1,\ldots,\pi_m)$, 

(a) $\rbcml{\mg}{\mw}$ is consistent if and only if $\rbcml{k_1\mg}{k_2\mw}$ is consistent for all $k_1, k_2>0$.

(b) If for any pair of alternatives $a_i, a_{i'}$ we have
\begin{equation}\label{proportionalx}
\frac{\bar\kappa_{ii'}}{\bar\kappa_{i'i}}=\frac{\Pr_{\vec\theta}(a_i\succ a_{i'})}{\Pr_{\vec\theta}(a_{i'}\succ a_i)}
\end{equation}
then $\rbcml{\mg}{\mw}$ is consistent if and only if $\mw$ is connected and symmetric.

(c) If $\mg$ has positive weight on an adjacent edge $l\ra l+1$, then $\rbcml{\mg}{\mw}$ is consistent only if $\mw$ is connected and symmetric.

2. For any  RUM$(\pi)$, 

(a) $\rbcml{\mg}{\mw}$ is consistent only if for any alternative $a_i$ we have 
\begin{equation}\label{eq:2ax}
\sum_{i'\ne i}w_{ii'}=\sum_{i'\ne i}w_{i'i}
\end{equation}
(b) Suppose the breaking graph contains an edge $\{l,l'\}$ that is different from $\{1,m\}$, then $\rbcml{\mg}{\mw}$ is consistent only if the $\mw$ is connected and symmetric.

(c) $\rbcml{\mg}{\mw}$ is consistent only if $\rbcml{\mg}{\mw_\text{u}}$ is consistent.

3. For any location family RUM$(\pi_1,\ldots,\pi_m)$ where each $\pi_i$ is symmetric around $0$, if $\rbcml{\mg}{\mw}$ is consistent, then $\rbcml{\mg}{\mw'}$ with symmetric weights $w'_{ii'}=w_{ii'}+w_{i'i}$ is also consistent.
\end{lem}

\begin{proof}

{\bf 1(a).} Let $\cll(\vec\theta, P)$ be the composite log-likelihood of $\rbcml{\mg}{\mw}$. Then the composite log-likelihood for $\rbcml{k_1\mg}{k_2\mw}$ is $k_1k_2\cll(\vec\theta, P)$. So if $\vec\theta^*$ maximizes $\cll(\vec\theta, P)$, it also maximizes $k_1k_2\cll(\vec\theta, P)$, or vice versa. That is to say, $\rbcml{\mg}{\mw}$ and $\rbcml{k_1\mg}{k_2\mw}$ are equivalent estimators. 

{\bf 1(b).} The ``if" direction: by combining \eqref{drzero} and \eqref{proportionalx}, the ground truth is the solution to \eqref{eqfirstorder}. Due to the strict concavity of $\cll(\vec\theta, P)$, the ground truth is the only solution. Consistency follows by Theorem~\ref{thm:asymptotic}.

The ``only if" direction: we first prove connectivity, then prove symmetry. 

If $\mw$ is not connected, then by Theorems~\ref{thm:logconcavepl} and \ref{thm:logconcaverum}, the solution to \eqref{eqfirstorder} is unbounded or non-unique. And by Theorem~\ref{thm:asymptotic}, $\rbcml{\mg}{\mw}$ is not consistent. 

Now we prove symmetry of $\mw$ by contradiction. For the purpose of contradiction suppose $w_{12} \neq w_{21}$ (w.l.o.g.). We will construct a counterexample where $\rbcml{\mg}{\mw}$ is not consistent. Let $\theta_1=\theta_2=0$ and $\theta_3=\ldots=\theta_m=L$. By Lemma~\ref{lem:infty}, we have for any $\epsilon>0$, there exists $L$ s.t. 
$
\nabla_1\exll(\vec\theta)=\frac {\bar\kappa_{12}w_{12}} {p_{12}(\vec\theta)}\frac {\partial p_{12}(\vec\theta)} {\partial\theta_1}+\frac {\bar\kappa_{21}w_{21}} {p_{21}(\vec\theta)}\frac {\partial p_{21}(\vec\theta)} {\partial\theta_1}+\epsilon=\frac {\bar\kappa_{21}(w_{21}-w_{12})} {p_{21}(\vec\theta)}\frac {\partial p_{21}(\vec\theta)} {\partial\theta_1}+\epsilon
$, 
where the last equality is obtained due to Lemma~\ref{lem:equal}. Since $w_{12}\neq w_{21}$, we have $\frac {\kappa_{21}(w_{21}-w_{12})} {p_{21}(\vec\theta)}\frac {\partial p_{21}(\vec\theta)} {\partial\theta_1}\neq 0$. Let $\epsilon<|\frac {\kappa_{21}(w_{21}-w_{12})} {p_{21}(\vec\theta)}\frac {\partial p_{21}(\vec\theta)} {\partial\theta_1}|$, then we have $\nabla_1\exll(\vec\theta)\neq 0$. This means the ground truth does not maximize $\exll(\vec\theta)$. By Theorem~\ref{thm:asymptotic}, the estimator is not consistent.

{\bf 1(c).} The proof for connectivity of $\mw$ is the same as in the proof of 1(b). We only prove necessity of symmetry.
For the purpose of contradiction suppose $w_{12}\neq w_{21}$. Let $\theta_1=\theta_2=0$, $\theta_3=\ldots=\theta_{l+1}=-L$, and $\theta_{l+2}=\ldots=\theta_m=L$. By Lemma~\ref{lem:infty}, for any $\epsilon>0$, we have 
$
\nabla_1\exll(\vec\theta) =\frac {\bar\kappa_{12}w_{12}} {p_{12}(\vec\theta)}\frac {\partial p_{12}(\vec\theta)} {\partial\theta_1}+\frac {\bar\kappa_{21}w_{21}} {p_{21}(\vec\theta)}\frac {\partial p_{21}(\vec\theta)} {\partial\theta_1}+\epsilon=\frac {\bar\kappa_{21}(w_{21}-w_{12})} {p_{21}(\vec\theta)}\frac {\partial p_{21}(\vec\theta)} {\partial\theta_1}+\epsilon
$, where the last equality is obtained by Lemma~\ref{lem:equal}. Since $w_{12}\neq w_{21}$, we have $\frac {\kappa_{21}(w_{21}-w_{12})} {p_{21}(\vec\theta)}\frac {\partial p_{21}(\vec\theta)} {\partial\theta_1}\neq 0$. Let $\epsilon<|\frac {\kappa_{21}(w_{21}-w_{12})} {p_{21}(\vec\theta)}\frac {\partial p_{21}(\vec\theta)} {\partial\theta_1}|$, then we have $\nabla_1\exll(\vec\theta)\neq 0$. This means the ground truth does not maximize $\exll(\vec\theta)$. By Theorem~\ref{thm:asymptotic}, the estimator is not consistent.

%Let $i-1$ alternatives be placed at $-L$ and $m-i-1$ alternatives be placed at $L$. In this case we have $$\lim_{L\ra \infty}\Pr_{\vec \theta}(a \text{ at }i\text{ and }b\text{ at }i+1|a\succ b)=\lim_{L\ra \infty}\Pr_{\vec \theta}(b \text{ at }i\text{ and }a\text{ at }i+1|b\succ a)=1$$
%This means that $w_{a,b}=w_{b,a}$.

{\bf 2(a).} Let $\theta_1=\ldots = \theta_m = 0$. Thus for any pair of alternatives $a_i, a_{i'}$, we have $\bar\kappa_{ii'}=\bar\kappa_{i'i}$ and $\Pr_{\vec\theta}(a_{i}\succ a_{i'})=\Pr_{\vec\theta}(a_{i'}\succ a_i)$. \eqref{eq:2ax} follows by applying \eqref{drzero} to $\exll(\vec\theta)=0$.

{\bf 2(b).} The proof for connectivity of $\mw$ is the same as in the proof of 1(b). For necessity of $\mw$, it suffices to prove $w_{12}=w_{21}$. Let $\Delta l=l'-l$ (w.l.o.g. suppose $l<l'$). Let $\theta_1=\ldots=\theta_{\Delta l+1}=0$, and $\theta_{\Delta l+2}=\ldots=\theta_{l+\Delta l}=L$, $\theta_{l'+1}=\ldots=\theta_m = -L$. When $L\ra +\infty$, with probability approaching $1$, $\theta_1$ through $\theta_{\Delta l+1}$ are ranked at positions from $l$ to $l'$. For any $1\le i, i' \le \Delta l+1$ and $i'\neq i$, we have $\bar\kappa_{ii'}=\bar\kappa_{i'i}$ and $\Pr_{\vec\theta}(a_{i}\succ a_{i'})=\Pr_{\vec\theta}(a_{i'}\succ a_i)$. So we have
\begin{equation}\label{2b1x}
\sum^{\Delta l+1}_{i=2} w_{1i}=\sum^{\Delta l+1}_{i=2} w_{i1}
\end{equation}
If we swap the values of $\theta_{\Delta l+2}$ and $\theta_{i'}$ where $2\le i'\le \Delta l+1$, we have
\begin{equation}\label{2b2x}
\sum^{\Delta l+2}_{i=2} w_{1i}-w_{1i'}=\sum^{\Delta l+2}_{i=2} w_{i1}-w_{i'1}
\end{equation}
Note that \eqref{2b2x} contains $\Delta l$ equations. Summing up all equations in \eqref{2b1x} and \eqref{2b2x}, we have 
\begin{equation}\label{2b3x}
\sum^{\Delta l+2}_{i=2}w_{1i}=\sum^{\Delta l+2}_{i=2}w_{i1}
\end{equation}
Let $i'=2$ in \eqref{2b2x}, we get
\begin{equation}\label{2b4x}
\sum^{\Delta l+2}_{i=3}w_{1i}=\sum^{\Delta l+2}_{i=3}w_{i1}
\end{equation}
\eqref{2b3x}-\eqref{2b4x}, we have $w_{12}=w_{21}$.

{\bf 2(c).} For any $\vec\theta$, $\nabla\exll(\vec\theta)=\vec 0$ holds. By relabeling the alternatives (by 
permuting the elements in $\vec\theta$), we can obtain $m!$ similar equations. Equivalently, each $w_{ii'}$ in $\mw$ will be the weight of $a_1\succ a_2$ (or any other pairwise comparison) for $(m-2)!$ times. By summing up all corresponding equations, we obtain another set of equations, which is the gradient of the composite likelihood with uniform $\mw'=(m-2)!\sum_{i\neq i'}w_{ii'}$. So $\rbcml{\mg}{\mw'}$ is also consistent. 

{\bf 3.} For any $\vec\theta$, we re-write \eqref{eqfirstorder}
\begin{equation}\label{eq31x}
\begin{split}
\nabla_i\exll(\vec\theta) = \sum_{i'\neq i}(\frac {\bar\kappa_{ii'}w_{ii'}} {p_{ii'}(\vec\theta)}\frac {\partial p_{ii'}(\vec\theta)} {\partial\theta_i}+\frac {\bar\kappa_{i'i}w_{i'i}} {p_{i'i}(\vec\theta)}\frac {\partial p_{i'i}(\vec\theta)} {\partial\theta_i})=0
\end{split}
\end{equation}
Consider the RUM with $\vec\theta'=-\vec\theta$, we have $p_{ii'}(\vec\theta')=p_{i'i}(\vec\theta)$. So we have
\begin{align}
&\nabla_i\exll(\vec\theta')= \sum_{i'\neq i}(\frac {\bar\kappa_{ii'}w_{ii'}} {p_{ii'}(\vec\theta')}\frac {\partial p_{ii'}(\vec\theta')} {\partial\theta_i'}+\frac {\bar\kappa_{i'i}w_{i'i}} {p_{i'i}(\vec\theta')}\frac {\partial p_{i'i}(\vec\theta')} {\partial\theta_i'})\notag\\
=& \sum_{i'\neq i}(-\frac {\bar\kappa_{ii'}w_{i'i}} {p_{ii'}(\vec\theta)}\frac {\partial p_{ii'}(\vec\theta)} {\partial\theta_i}-\frac {\bar\kappa_{i'i}w_{ii'}} {p_{i'i}(\vec\theta)}\frac {\partial p_{i'i}(\vec\theta)} {\partial\theta_i})=0\label{eq32x}
\end{align}
\eqref{eq31x}-\eqref{eq32x}, we have
$
\sum_{i'\neq i}(\frac {\kappa_{ii'}(w_{ii'}+w_{i'i})} {p_{ii'}(\vec\theta)}\frac {\partial p_{ii'}(\vec\theta)} {\partial\theta_i}+\frac {\kappa_{i'i}(w_{i'i}+w_{i'i})} {p_{i'i}(\vec\theta)}\frac {\partial p_{i'i}(\vec\theta)} {\partial\theta_i})=0
$, which means  $\rbcml{\mg}{\mw'}$ is consistent by Theorem~\ref{thm:asymptotic}.
\end{proof}

%%%%%%%%%%Theorems
%Theorem 1

{\bf Theorem~\ref{thm:logc_conv}}
Let $f(x)$ and $g(x)$ be two continuous and strictly log-concave functions on $\mathbb R$. Then $f*g$ is also strictly log-concave on $\mathbb R$.

\begin{proof}
The proof is done by examining the equality condition for the Pr\'ekopa-Leindler inequality. %which is the key step in the proof of the preservation of log-concavity under marginalization. 
Let $h=f*g$, namely, for any $y\in \mathbb R$, $h(y)=\int_{\mathbb R} f(y-x)g(x) dx$.
Because $f$ and $g$ are continuous, so does $h$. To prove the strict log-concavity of $h$, 
it suffices to prove that for any different $y_1,y_2\in\mathbb R$,  $h(\frac{y_1+y_2}{2})>\sqrt {h(y_1)h(y_2)}$.

Suppose for the sake of contradiction that this is not true. Since log-concavity preserves under convolution \citep{Saumard2014:Log-concavity}, $h$ is log-concave. So, there exist $y_1<y_2$ such that $h(\frac{y_1+y_2}{2})=\sqrt {h(y_1)h(y_2)}$. Let $\Lambda(x,y)=f(y-x)g(x)$. %It follows that $H$ is a log-concave function on ${\mathbb R}^2$. 
We further define 
\begin{align*}
H(x)&=\Lambda(x,\frac{y_1+y_2}{2})=f(\frac{y_1+y_2}{2}-x)g(x)\\
F(x)&=\Lambda(x,y_1)=f(y_1-x)g(x)\\
G(x)&=\Lambda(x,y_2)=f(y_2-x)g(x)
\end{align*}
Because (non-strict) log-concavity is preserved under convolution, $\Lambda(x,y)$ is log-concave. We have that for any $x\in\mathbb R$, $H(x)\ge \sqrt{F(x)G(x)}$. The Pr\'ekopa-Leindler inequality asserts that 
\begin{equation}\label{eq:plx}
\int_{\mathbb R}H(x) dx\ge \sqrt {\int_{\mathbb R}F(x) dx\int_{\mathbb R}G(x) dx}
\end{equation}

Because $h(\frac{y_1+y_2}{2})=\int_{\mathbb R}H(x) dx$, $h(y_1)=\int_{\mathbb R}F(x) dx$, $h(y_2)=\int_{\mathbb R}G(x) dx$, and  $h(\frac{y_1+y_2}{2})=\sqrt {h(y_1)h(y_2)}$, (\ref{eq:plx}) becomes an equation. It was proved by~\citet{Dubuc1977:Critere} that: there exist $a>0$ and $b\in \mathbb R$ such that the following conditions hold almost everywhere for $x\in \mathbb R$ (see the translation of Dubuc's result in English by~\citet{Ball2010:Stability}). 1.~$F(x)=aH(x+b)$, 2.~$G(x)=a^{-1}H(x-b)$.

The first condition means that for almost every $x\in\mathbb R$,
\begin{align}
&f(y_1-x)g(x)=af(\frac{y_1+y_2}{2}-x-b)g(x+b)\notag\\
&\Longleftrightarrow \frac{g(x)}{g(x+b)}=a \frac{f(\frac{y_1+y_2}{2}-x-b)}{f(y_1-x)}\label{eq:eqc1x}
\end{align}
The second condition means that for almost all  $x\in\mathbb R$,
$
f(y_2-x)g(x)=a^{-1}f(\frac{y_1+y_2}{2}-x+b)g(x-b)
\Longleftrightarrow \frac{g(x-b)}{g(x)}=a \frac{f(y_2-x)}{f(\frac{y_1+y_2}{2}-x+b)}
$. Therefore, for almost all $x\in \mathbb R$, 

\begin{equation}\label{eq:eqc3x}\frac{g(x)}{g(x+b)}=a \frac{f(y_2-x-b)}{f(\frac{y_1+y_2}{2}-x)}\end{equation}

Combining (\ref{eq:eqc1x}) and (\ref{eq:eqc3x}), for almost every $x\in \mathbb R$ we have
\begin{equation}\label{eq:eqc4x}
\frac{g(x)}{g(x+b)}=a \frac{f(y_2-x-b)}{f(\frac{y_1+y_2}{2}-x)}=a \frac{f(\frac{y_1+y_2}{2}-x-b)}{f(y_1-x)}
\end{equation}
Because $f(x)$ is strictly log-concave, for any fixed $c\ne 0$, $\frac{f(x+c)}{f(x)}$ is strictly monotonic. Because $y_1\ne y_2$ and 
$y_2-x-b-(\frac{y_1+y_2}{2}-x)=\frac{y_1+y_2}{2}-x-b-(y_1-x)=\frac{y_2-y_1}{2}-b$, we must have that $\frac{y_2-y_1}{2}-b=0$, namely $b=\frac{y_2-y_1}{2}$. Therefore, (\ref{eq:eqc4x}) becomes $\frac{g(x)}{g(x+\frac{y_2-y_1}{2})}=a$ for almost every $x\in\mathbb R$, which contradicts the strict log-concavity of $g$. This means that $h=f*g$ is strictly log-concave.
\end{proof}

%Theorem 2
{\bf Theorem~\ref{thm:logconcavemarginal}}
Let $h(x,y)$ be a strictly log-concave function on $\mathbb R^2$. Then $\int_{\mathbb R}h(x,y) dx$ is strictly log-concave on $\mathbb R$.

\begin{proof} Again, the proof is done by examining the equality condition for the Pr\'ekopa-Leindler inequality. Let $h^*(y)=\int_{\mathbb R} h(x,y) dx$. It suffices to prove that for any different $y_1,y_2\in\mathbb R$,  $h^*(\frac{y_1+y_2}{2})>\sqrt {h^*(y_1)h^*(y_2)}$. 

Suppose for the sake of contradiction the claim is not true. Because (non-strict) log-concavity is preserved under marginalization, $h^*$ is log-concave. Therefore, there exist $y_1<y_2$ such that $h^*(\frac{y_1+y_2}{2})=\sqrt {h^*(y_1)h^*(y_2)}$. We further define the following functions. $H(x)=h(x,\frac{y_1+y_2}{2})$,  $F(x)=h(x,y_1)$, and  $G(x)=h(x,y_2)$.

Because $h(x,y)$ is strictly log-concave, we have that for any $x\in\mathbb R$, $H(x)> \sqrt{F(x)G(x)}$. The Pr\'ekopa-Leindler inequality asserts that 
\begin{equation}\label{eq:plx2}
\int_{\mathbb R}H(x) dx\ge \sqrt {\int_{\mathbb R}F(x) dx\int_{\mathbb R}G(x) dx}
\end{equation}

Because $h^*(\frac{y_1+y_2}{2})=\int_{\mathbb R}H(x) dx$, $h^*(y_1)=\int_{\mathbb R}F(x) dx$, $h^*(y_2)=\int_{\mathbb R}G(x) dx$, and  $h^*(\frac{y_1+y_2}{2})=\sqrt {h^*(y_1)h^*(y_2)}$,  (\ref{eq:plx2}) becomes an equation. Following~\citet{Dubuc1977:Critere}'s result, we have that there exist $a>0$ and $b\in \mathbb R$ such that $F(x)=aH(x+b)$ and $G(x)=a^{-1}H(x-b)$ hold almost everywhere for $x\in \mathbb R$.

$F(x)=aH(x+b)$ means that for almost every $x\in\mathbb R$,
$ah(x+b,\frac{y_1+y_2}{2})=h(x,y_1)$. 
$G(x)=a^{-1}H(x-b)$ means that for almost every  $x\in\mathbb R$, $a^{-1}h(x-b,\frac{y_1+y_2}{2})=h(x,y_2)$. This means that for almost every  $x\in\mathbb R$, $
a^{-1}h(x+b,\frac{y_1+y_2}{2})=h(x+2b,y_2)
$. Therefore, for almost every $x\in \mathbb R$, we have $h(x+b,\frac{y_1+y_2}{2})\cdot h(x+b,\frac{y_1+y_2}{2})=h(x,y_1)\cdot h(x+2b,y_2)$, which contradicts the strict log-concavity of $h$.
\end{proof}

%Theorem 3
{\bf Theorem~\ref{thm:logconcavepl}} Given any profile $P$, the composite likelihood function for Plackett-Luce, i.e.~$\clpl(\vec\theta,P)$, is strictly log-concave if and only if $\mw$ is weakly connected.  $\arg\max_{\vec\theta}\clpl(\vec\theta,P)$ is bounded if and only if $\mw\otimes G(P)$ is strongly connected.

\begin{proof} It is not hard to check that when $\mw$ is not weakly connected, there exist $\vec\theta^{(1)}$ and $\vec\theta^{(2)}$ such that for any $0<\lambda<1$ we have $\cllpl(\vec\theta^{(1)},P)=\cllpl(\vec\theta^{(2)},P)=\lambda\cllpl(\vec\theta^{(1)},P)+(1-\lambda)\cllpl(\vec\theta^{(2)},P)$, which violates strict log-concavity.

Suppose $\mw$ is weakly connected, we only need to show that
\begin{equation}\label{cllpl_simx}
f(\vec\theta) = \sum_{i_1\neq i_2}(-(\kappa_{i_1i_2}w_{i_1i_2}+\kappa_{i_2i_1}w_{i_2i_1})\ln(e^{\theta_{i_1}}+e^{\theta_{i_2}}))
\end{equation}
is concave. The proof is similar to the log-concavity of likelihood for BTL by~\cite{Hunter04:MM}. H\"{o}lder's inequality shows that for positive $c_t, d_t>0$, where $t=1, \ldots, N$ and $0<\lambda<1$, we have
\begin{equation}\label{holderx}
\ln\sum^N_{t=1}c^\lambda_td^{1-\lambda}_t\leq \lambda\ln\sum^N_{t=1}c_t+(1-\lambda)\ln\sum^N_{t=1}d_t
\end{equation}
with equality if and only if $\exists\zeta$ s.t. $c_t=\zeta d_t$ for all $t$.

Let $\vec{\theta}^{(1)}$ and $\vec{\theta}^{(2)}$ be two parameters. For any two alternatives $a_{i_1}$ and $a_{i_2}$, by \eqref{holderx}, we have
\begin{align*}
-\ln(e^{\lambda\theta^{(1)}_{i_1}+(1-\lambda)\theta^{(2)}_{i_1}}+e^{\lambda\theta^{(1)}_{i_2}+(1-\lambda)\theta^{(2)}_{i_2}}\geq -\lambda\ln(e^{\theta^{(1)}_{i_1}}+e^{\theta^{(1)}_{i_2}})-(1-\lambda)\ln(e^{\theta^{(2)}_{i_1}}+e^{\theta^{(2)}_{i_2}})
\end{align*}
Multiplying both sides by $\kappa_{i_1i_2}w_{i_1i_2}+\kappa_{i_2i_1}w_{i_2i_1}$ and summing over all $i_i\neq i_2$ demonstrates the concavity of \eqref{cllpl_simx}.

To prove strict concavity, we need to check the condition when the equality of \eqref{holderx} holds. For all $1\leq i\leq m$, $e^{\theta^{(1)}_i} = \zeta e^{\theta^{(2)}_i}$. Namely $\theta^{(1)}_i = \theta^{(2)}_i+\ln\zeta$ holds for all $i$. Because random utility models are invariant under parameter shifts, it is exactly the same model. Thus, we proves the strict concavity of \eqref{cllpl_simx}.

The proof for the condition of boundedness is also similar to that by~\citet{Hunter04:MM}.
\end{proof}

%Theorem 4
{\bf Theorem~\ref{thm:logconcaverum}}
Let $\mm$ be an RUM where the CDF of each utility distribution is strictly log-concave. 
Given any profile $P$, the composite likelihood function for $\mm$, i.e.~$\cl(\vec\theta,P)$, is strictly log-concave if and only if $\mw$ is weakly connected.  $\arg\max_{\vec\theta}\cl(\vec\theta,P)$ is bounded if and only if $\mw\otimes G(P)$ is strongly connected.

\begin{proof} Similar to the proof for Plackett-Luce, the only hard part is to prove that when $\mw$ is weakly connected, $\cl(\vec\theta,P)$ is strictly log-concave. It suffice to  prove for any $i_1\neq i_2$, $\Pr(a_{i_1}\succ a_{i_2}|\vec\theta)$ is log-concave, namely $\Pr(u_{i_1}>u_{i_2}|\vec\theta)$ is log-concave. We can write this probability as integral over $u_{i_2}-u_{i_1}$: 
$
\Pr(u_{i_1}>u_{i_2}|\vec\theta) = \int^\infty_0\Pr(u_{i_2}-u_{i_1}=s|\vec\theta)ds
$.

Let $\pi^\ast_{i_2}(\cdot|\vec\theta)$ denote the flipped distribution of $\pi_{i_2}(\cdot|\vec\theta)$ around $x=s$, then we have $\pi^*_{i_2}(s-x|\vec\theta)=\pi_{i_2}(s+x|\vec\theta)$. Therefore we have
$$
\Pr(u_{i_1}>u_{i_2}|\vec\theta) = \int^\infty_0\int^\infty_{-\infty} \pi_{i_1}(x|\theta_{i_1})\pi_{i_2}(x+s|\theta_{i_2})dxds
= \int^\infty_0\int^\infty_{-\infty} \pi_{i_1}(x|\theta_{i_1})\pi^\ast_{i_2}(s-x|\theta_{i_2})dxds
= \int^\infty_0 \pi_{i_1} * \pi_{i_2}^\ast ds
$$

By Theorem~\ref{thm:logc_conv} we know $\pi_{i_1} * \pi_{i_2}^\ast$ is strictly log-concave. We only need to prove that tail probability of a strictly log-concave distribution is also strictly log-concave, which is shown in Lemma~\ref{lem:tailconc}. 
\end{proof}

%Theorem 5

{\bf Theorem~\ref{thm:asymptotic}}
Given any RUM $\mm$, any $\vec\theta_0$ and any profile $P$ with $n$ rankings. Let $\vec\theta^*$ be the output of $\rbcml{\mg}{\mw}$. When $n\rightarrow\infty$, we have  $\vec\theta^*\xrightarrow{p}\vec\theta_0$ and
$\sqrt{n}(\vec\theta^*-\vec\theta_0)\xrightarrow{d}N(0, H^{-1}_0(\vec\theta_0)\text{Var}[\nabla \cll(\vec\theta_0, R)]H^{-1}_0(\vec\theta_0))$ if and only if $\vec\theta_0$ is the only solution to 
\begin{equation}\label{eqfirstorderx}
\nabla \exll(\vec\theta)=\vec 0
\end{equation}
 
\begin{proof}
The ``only if" direction is straightforward. The solution to \eqref{eqfirstorderx} is unique because $\cll(\vec\theta, P)$ is strictly concave. Suppose $\vec\theta_1$, other than $\vec\theta_0$, is the solution to \eqref{eqfirstorderx}, then when $n\ra\infty$, $\vec\theta_1$ will be the estimate of $\rbcml{\mg}{\mw}$, which means $\rbcml{\mg}{\mw}$ is not consistent.

Now we prove the ``if" direction. First we prove consistency. {It is required by \citet{Xu2011:On-the-robustness} that for different parameters, the probabilities for any composite likelihood event are different, which is not true in our case. A simple counterexample is $\theta^{(1)}_1=1, \theta^{(2)}_1=2, \theta^{(1)}_2=\theta^{(1)}_3=\theta^{(2)}_2=\theta^{(2)}_3=0$. Then $\Pr(a_2\succ a_3|\vec\theta^{(1)})=\Pr(a_2\succ a_3|\vec\theta^{(2)})$.}
%\lirong{discuss which condition it violoates for~\citet{Xu2011:On-the-robustness}}

By the law of large numbers, we have for any $\epsilon$, $\Pr(|\cll( \vec\theta,P)-\exll(\vec\theta)|\leq \epsilon/2)\rightarrow 1$ as $n\rightarrow\infty$. This implies $\lim_{n\ra\infty}\Pr(\cll( \vec\theta^*,P)\leq \exll(\vec\theta^*)+\epsilon/2)=1$. Similarly we have $\lim_{n\ra\infty}\Pr(\exll(\vec\theta_0)\leq \cll( \vec\theta_0,P)+\epsilon/2)=1$. Since $\vec\theta^*$ maximize $\cll( \vec\theta,P)$, we have $\Pr(\cll( \vec\theta_0,P)\leq \cll( \vec\theta^*,P))=1$.
%\end{equation}
The above three equations imply that $\lim_{n\ra\infty}\Pr(\exll(\vec\theta_0)-\exll(\vec\theta^\ast)\leq \epsilon)=1$.

Let $\Theta_\epsilon$ be the subset of parameter space s.t. $\forall \vec\theta\in\Theta_\epsilon$, $\exll(\vec\theta_0)-\exll(\vec\theta)\leq \epsilon$. Because $\exll(\vec\theta)$ is strictly concave, $\Theta_\epsilon$ is compact and has a unique maximum at $\vec\theta_0$. Thus for any $\epsilon>0$, $\lim_{n\ra\infty}\Pr(\vec\theta^*\in\Theta_\epsilon)=1$. This implies consistency, i.e., $\vec\theta^*\xrightarrow{p}\vec\theta_0$.

Now we prove asymptotic normality. By mean value theorem, we have
%\begin{align*}
$0=\nabla \cll( \vec\theta^*,P)
=\nabla \cll( \vec\theta_0,P)+H(\alpha\vec\theta^*+(1-\alpha)\vec\theta_0, P)(\vec\theta^*-\vec\theta_0)$, %\end{align*}
where $0\leq \alpha\leq 1$. Therefore, we have
$\sqrt n(\vec\theta^*-\vec\theta)
=-H^{-1}(\alpha\vec\theta^*+(1-\alpha)\vec\theta_0, P)(\sqrt n\nabla \cll( \vec\theta_0,P))
$. 
Since $\nabla \cll( \vec\theta_0,P)=\frac 1 n\sum^n_{j=1}\nabla \cll(\vec\theta_0, R_j)$, by the central limit theorem, we have

$\hfill
\sqrt n\nabla \cll( \vec\theta_0,P)\xrightarrow{d} N(0, \text{Var}[\nabla \cll(\vec\theta_0, R)])
\hfill$

Because $\vec\theta^*\xrightarrow{p}\vec\theta_0$ and $H$ is continuous, we have $H(\alpha\vec\theta^*+(1-\alpha)\vec\theta_0, P)\xrightarrow{p} H(\vec\theta_0,P)$. Since $H(\vec\theta, P)=\frac 1 n\sum^n_{j=1}H(\vec\theta, R_j)$, by law of large numbers, we have $H(\vec\theta, P)\xrightarrow{p} H_0(\vec\theta_0)$. Therefore, we have
$$\sqrt n(\vec\theta^*-\vec\theta)=-H^{-1}_0(\vec\theta_0)(\sqrt n\nabla \cll( \vec\theta_0,P)),$$
which implies that $\text{Var}[\sqrt n(\vec\theta^*-\vec\theta)]=H_0^{-1}(\vec\theta_0)\text{Var}[\nabla \cll(\vec\theta_0, R)]H_0^{-1}(\vec\theta_0).$
\end{proof}

{\bf Theorem~\ref{thm:pl}}
$\rbcml{\mg}{\mw_{\text u}}$ is consistent for Plackett-Luce  if and only if the breaking is weighted union of position-$k$ breaking.

\begin{proof}
The ``if" direction is proved in \cite{Khetan16:Data}. We only prove the ``only if" direction.

We will prove this theorem by induction on $m$. When $m=2$, the only breaking is the comparison between the two alternatives. The conclusion holds. Suppose it holds for $m=l$, then when $m=l+1$, we first apply Lemma 2 to $\mg_{[2, m]}$, which must be a weighted union of position-$k$ breaking. Then apply Lemma 2 to $\mg_{[1, m-1]}$. For all $i\leq m-1$, $g_{1i}$ are the same, denoted by $g_0$. We claim that $g_{1m}=g_0$. The reason is as follows.

For the purpose of contradiction suppose $g_{1m}\neq g_0$. If $g_{1m}>g_0$. We split this edge into two parts, one with weight $g_0$ and the other $g_{1m}-g_0$. Let $\mg_1=\{g_{1m}=g_0\}\cup(\mg-g_{1m})$ , and $\mg_2=\{g_{1m}=g-g_0\}$. So we have $\mg=\mg_1+\mg_2$. Because $\rbcml{\mg_1}{\mw_{\text{u}}}$ is consistent and $\rbcml{\mg_2}{\mw_{\text{u}}}$ is not (Lemma~\ref{pl1m}). By Lemma~\ref{lem:plusminus}, $\rbcml{\mg}{\mw_{\text{u}}}$ is not consistent, which is a contradiction. The case where $g<g_0$ is similar.
\end{proof}

{\bf Theorem~\ref{thm:rum}}
Let $\pi_1, \pi_2, \ldots, \pi_m$ denote the utility distributions for a symmetric RUM. Suppose there exists $\pi_i$ s.t. $(\ln\pi_i(x))'$ is monotonically decreasing and $\lim_{x\ra -\infty}(\ln\pi_i(x))'\ra \infty$. $\rbcml{\mg}{\mw}$ is consistent if and only if $\mg$ is uniform.

\begin{proof}
We prove the theorem by induction on $m$. $m=2$ is trivial because the only breaking is uniform. For $m=3$ we know the uniform breaking is consistent and the one-edge breaking $\mg=\{g_{13}=C>0\}$ is not consistent by Lemma~\ref{lem:rum1m}. Suppose the breaking is $\mg=\{g_{12}=x, g_{23} = y, g_{13} = z\}$.

{\bf Case 1:} $x+y\neq 2z$. For the sake of contradiction suppose $\rbcml{\mg}{\mw_{\text{u}}}$ is consistent. By Lemma~\ref{lem:flip}, $\rbcml{\mg^*}{\mw_{\text{u}}}$ is consistent for $\mm^\ast$, which is $\mm$ due to the symmetry of utility distributions. Applying Lemma~\ref{lem:plusminus} we have $\rbcml{\mg+\mg^*}{\mw_{\text{u}}}$ is consistent, where $\mg+\mg^*=\{g_{12}=x+y, g_{23} = x+y, g_{13} = 2z\}$. If $x+y<2z$, we have $\rbcml{\mg+\mg^*-(x+y)\mg_{\text{u}}}{\mw_{\text{u}}}$ is consistent, where $\mg+\mg^*-(x+y)\mg_{\text{u}}=\{g_{1m}=2z-x-y\}$. This contradicts Lemma~\ref{lem:rum1m}. The case with $x+y>2z$ is similar.

%Now we prove any one-edge breaking is not consistent. For the purpose of contradiction suppose $B_{\{\{1, 2\}\}}$ is consistent. By Lemma~\ref{lem:flip} $B_{\{\{1, 2\}\}^\ast}=B_{\{\{2, 3\}\}}$ is consistent for $\mm^\ast$, which is $\mm$ due to the symmetry of utility distributions. Then we have the uniform breaking is not consistent due to Lemma~\ref{lem:plusminus}, which is a contradiction. For similar reason $B_{\{\{2, 3\}\}}$ is not consistent. 
%
%Now we prove any weighted combination of two-edge breaking is not consistent. Suppose $B_{\{x\{1,2\},y\{2,3\}\}}$ is consistent, then by Lemma~\ref{lem:flip} $B_{\{x\{1,2\},y\{2,3\}\}^\ast}=B_{\{x\{2,3\},y\{1,2\}\}}$ is also consistent. By Lemma~\ref{lem:plusminus} $B_{\{(x+y)\{1,2\},(x+y)\{2,3\}\}}$ is consistent. Because uniform breaking is consistent, we know $B_{\{(x+y)\{1, 3\}\}}$ is consistent by Lemma~\ref{lem:plusminus}, which is a contradiction. Similarly neither $B_{\{x\{1,2\},y\{1,3\}\}}$ nor $B_{\{x\{1,3\},y\{2,3\}\}}$ is consistent.

{\bf Case 2:} x+y=2z. Lemma~\ref{lem:g210} states that $\rbcml{\mg_{210}}{\mw_{\text u}}$ is not consistent where $\mg_{210}=\{g_{12}=2, g_{13} = 1\}$. We have $\mg=y\mg_{\text u}+(z-y)\mg_{210}$. Since any $\mg_{\text u}$ is consistent, $\rbcml{\mg}{\mw_{\text u}}$ is not consistent.

%Let $\theta_1\neq\theta_2=\theta_3=0$. Further let
%\begin{align*}
%\Pr(a_1\succ a_2\succ a_3) &= \Pr(a_1\succ a_3\succ a_2) = p_1\\
%\Pr(a_2\succ a_1\succ a_3) &= \Pr(a_3\succ a_1\succ a_2) = p_2\\
%\Pr(a_2\succ a_3\succ a_1) &= \Pr(a_3\succ a_2\succ a_1) = p_3
%\end{align*}
%Then we have $\Pr(a_1\succ a_2)=2p_1+p_2$, $\Pr(a_2\succ a_1)=p_2+2p_3$
%\begin{equation}\label{sumhalf}
%p_1+p_2+p_3=\frac 1 2
%\end{equation}
%\eqref{rumdr} becomes
%\begin{align*}
%\frac {\partial \cll(\vec\theta)} {\partial\theta_1} &= \sum_{i=2, 3}(\frac {\kappa_{1i}} {p_{1i}(\vec\theta)}\frac {\partial p_{1i}(\vec\theta)} {\partial\theta_1}+\frac {\kappa_{i1}} {p_{i1}(\vec\theta)}\frac {\partial p_{i1}(\vec\theta)} {\partial\theta_1})\\
%&= 2(\frac {\kappa_{12}} {p_{12}(\vec\theta)}\frac {\partial p_{12}(\vec\theta)} {\partial\theta_1}+\frac {\kappa_{21}} {p_{21}(\vec\theta)}\frac {\partial p_{21}(\vec\theta)} {\partial\theta_1})\\
%&= 2\frac {\partial p_{12}(\vec\theta)} {\partial\theta_1}(\frac {\frac 3 2 p_1} {2p_1+p_2}-\frac {p_2+\frac 1 2 p_3} {p_2+2p_3})=0
%\end{align*}
%We have
%\begin{equation}\label{condconsist}
%\frac 3 2 p_1(p_2+2p_3)=(2p_1+p_2)(p_2+\frac 1 2 p_3)
%\end{equation}
%Combining \eqref{sumhalf} and \eqref{condconsist} we have $p_2=\sqrt{3p_3+\frac 1 {16}}-2p_3-\frac 1 4$ and $p_1=\frac 1 2 -p_2-p_3$. Let $\theta_1=2$ and standard deviations of all utility distributions be $1$, then we have $p_1=0.4329, p_2=0.0556, p_3=0.0115$, where \eqref{condconsist} does not hold. So the breaking for RUM-Gaussian is not consistent.

Suppose the theorem holds for $m=k$. When $m=k+1$, W.l.o.g. we let $\pi_2$ satisfy the conditions that $(\ln\pi_i(x))'$ is monotonically decreasing and $\lim_{x\ra -\infty}(\ln\pi_i(x))'\ra \infty$. Let $\theta_1=L$, $\theta_m=-L$, and $\theta_2=\ldots=\theta_{m-1}=0$. So when $L\ra\infty$, with probability that goes to $1$, $a_1$ is ranked at the top and $a_m$ is ranked at the bottom. Let $\mg_{\{1m\}}=\{g_{1m}=1\}$. We apply Lemma~\ref{lem:subgraph} to $\mg_{[2,m]}$ and $\mg_{[1, m-1]}$. By induction hypothesis $\mg_{[2,m]}$ (or $\mg_{[1, m-1]}$) is uniform breaking graph or empty. If $\mg_{[2,m]}$ is empty, then $\mg_{[1, m-1]}$ is also empty. As $\mg$ is nonempty, $\mg=C\mg_{\{1m\}}$, which contradicts Lemma~\ref{lem:rum1m}. If $\mg_{[2,m]}$ is uniform. We denote the weight as $g_0$. Then $\mg_{[1, m-1]}$ is also uniform with weight $g_0$. Then the only consistent breaking is uniform. The reason is as follows. We can write $\mg=g_0\mg_{\text u}+(g_{1m}-g_0)\mg_{\{1m\}}$. By Lemma~\ref{lem:rum1m} and Lemma~\ref{lem:plusminus}, $\rbcml{\mg}{\mw}$ is not consistent, which is a contradiction.
\end{proof}

{\bf Theorem~\ref{thm:cmlrbconsistencypl}}
$\rbcml{\mg}{\mw}$ for Plackett-Luce is consistent if and only if $\mg$ is the weighted union of position-$k$ breakings and $\mw$ is connected and symmetric.

\begin{proof}
The ``only if" direction: 2(c) part of the Lemma~\ref{lem:consistency} states that if $\rbcml{\mg}{\mw}$ is consistent then $\rbcml{\mg}{\mw_{\text{u}}}$ is consistent, which means that $\mg$ is the weighted union of position-$k$ breakings by Theorem~\ref{thm:pl}. Then following 1(c) part of the Lemma~\ref{lem:consistency}, $\mw$ must be connected and symmetric. 

The ``if" direction: $\mg$ is the weighted union of position-$k$ breakings. For any $a_i$, $a_{i'}$, we have $\sum_{i'\neq i}(\bar\kappa_{ii'}-(\bar\kappa_{ii'}+\bar\kappa_{i'i})\frac {e^{\theta_i}} {e^{\theta_i}+e^{\theta_{i'}}})=0$. Because $w_{ii'}=w_{i'i}$, we have
$
\nabla_i\exll(\vec\theta)= \sum_{i'\neq i}(\bar\kappa_{ii'}w_{ii'}-(\kappa_{ii'}w_{ii'}+\bar\kappa_{i'i}w_{i'i})\frac {e^{\theta_i}} {e^{\theta_i}+e^{\theta_{i'}}})=0
$. 
This means the ground truth is the solution to $\nabla\exll(\vec\theta)=\vec 0$. As $\mw$ is connected and symmetric, it is strongly connected. Thus $\cllpl$ is strictly concave, which means the ground truth is the only solution. Further by Theorem~\ref{thm:asymptotic}, $\rbcml{\mg}{\mw}$ is consistent.
\end{proof}

{\bf Theorem~\ref{thm:cmlrbconsistencyrum}}
Let $\pi$ be any symmetric distribution that satisfies the condition in Theorem~\ref{thm:rum}. Then $\rbcml{\mg}{\mw}$ is consistent for RUM$(\pi)$ if and only if $\mg$ is uniform and $\mw$ is connected and symmetric.

\begin{proof} 
The ``only if" direction: 2(c) part of Lemma~\ref{lem:consistency} states that $\rbcml{\mg}{\mw}$ is consistent with uniform $\mw$, which implies $\mg$ must be uniform by Theorem~\ref{thm:rum}. Then 1(c) of Lemma~\ref{lem:consistency} implies that $\rbcml{\mg}{\mw}$ is consistent for any connected and symmetric $\mw$. 

The ``if" direction: Since $\mg$ is uniform breaking, we have $\sum_{i'\neq i}(\frac {\bar\kappa_{ii'}} {p_{ii'}(\vec\theta)}\frac {\partial p_{ii'}(\vec\theta)} {\partial\theta_i}+\frac {\bar\kappa_{i'i}} {p_{i'i}(\vec\theta)}\frac {\partial p_{i'i}(\vec\theta)} {\partial\theta_i})=0$
Because $w_{ii'}=w_{i'i}$, we have
\begin{align*}
&\nabla_i\exll(\vec\theta)=\sum_{i'\neq i}(\frac {\bar\kappa_{ii'}w_{ii'}} {p_{ii'}(\vec\theta)}\frac {\partial p_{ii'}(\vec\theta)} {\partial\theta_i}+\frac {\bar\kappa_{i'i}w_{i'i}} {p_{i'i}(\vec\theta)}\frac {\partial p_{i'i}(\vec\theta)} {\partial\theta_i})=0
\end{align*}
holds for all $i$. This means the ground truth is the solution to $\nabla\exll(\vec\theta)=\vec 0$. As $\mw$ is connected and symmetric, it is strongly connected. Thus $\cll$ is strictly concave, which means the ground truth is the only solution. Further by Theorem~\ref{thm:asymptotic}, $\rbcml{\mg}{\mw}$ is consistent.
\end{proof}

%}
\end{document}